\documentclass[11pt]{article}
\usepackage{macros}
\usepackage{fullpage}
\usepackage{mathtools,amsfonts,amsthm,amssymb,bbm}
\usepackage{algorithm}
\usepackage{algorithmicx}
\usepackage[noend]{algpseudocode}
\usepackage{xcolor}
\usepackage[nameinlink]{cleveref}

\usepackage{color-edits}
\addauthor{mq}{blue}
\addauthor{jl}{red}

\title{Limitations of Membership Queries in Testable Learning}
\author{Jane Lange\thanks{MIT, {\tt jlange@mit.edu}.} \and Mingda Qiao\thanks{UMass Amherst, {\tt mqiao@umass.edu}. }}

\date{\vspace{15pt}\small{\today}}

\begin{document}

\maketitle

\begin{abstract}
Membership queries (MQ) often yield speedups for learning tasks,
particularly in the distribution-specific setting.
We show that in the \emph{testable learning} model of Rubinfeld and Vasilyan \cite{RV23},
membership queries cannot decrease the time complexity of testable learning algorithms
beyond the complexity of sample-only distribution-specific learning.
In the testable learning model, the learner must output a 
hypothesis whenever the data distribution satisfies a desired property, and if it 
outputs a hypothesis, the hypothesis must be near-optimal.

We give a general reduction from sample-based \emph{refutation} of boolean concept classes, as presented in \cite{Vadhan17, KL18}, 
to testable learning with queries (TL-Q).
This yields lower bounds for TL-Q via the reduction
from learning to refutation given in \cite{KL18}. 
The result is that, relative to a concept class and a distribution family,
no $m$-sample TL-Q algorithm 
can be
super-polynomially
more time-efficient than the best
$m$-sample PAC learner.

Finally, we define a class of ``statistical'' MQ algorithms that encompasses 
many known distribution-specific MQ learners, such as those based on influence estimation or subcube-conditional statistical queries.
We show that TL-Q algorithms in this class
imply efficient statistical-query refutation and learning algorithms.
Thus, combined with known SQ dimension lower bounds, our results imply that these efficient membership query learners cannot be made testable.
\end{abstract}

\section{Introduction}\label{sec:intro}
In distribution-specific PAC learning, a learning algorithm is only required to output a competitive hypothesis when the data distribution satisfies some property.
Distribution-specific PAC often allows for much more efficient learning than distribution-free PAC, 
but with the following shortcoming: 
if the distribution does not satisfy the property,
then the behavior of the learner is completely undefined.

A \emph{testable agnostic learning} algorithm \cite{RV23} alleviates this shortcoming by combining a distribution-specific learner with a tester for the desired property.
It may output a hypothesis or it may reject the distribution and output $\bot$.
The testable learner is run on i.i.d.\ samples from an unknown distribution $\mcD$ over
$\mcX \times \zo$, and has the following behavior:

\begin{itemize}
		\item \textbf{Soundness:} If the learner outputs a hypothesis $h$, then with high probability
\[\Pr_{(x,y) \sim \mcD}[h(x) \ne y] \le \opt + \eps,\]
where $\opt$ is the error of the best concept in the concept class. A \emph{semi-agnostic} variant of this condition, with $\opt + \eps$ replaced by $O(\opt) + \eps$, has also been considered.
\item \textbf{Completeness:} If the distribution $\mcD$ has the desired property, then the learner outputs a hypothesis (instead of $\bot$) with high probability.
\end{itemize}

There is a significant body of work studying the sample and time complexities of testable learning for various concept classes and distribution properties in both the agnostic ($\opt + \eps$) and semi-agnostic ($O(\opt) + \eps$) settings \cite{RV23, GKK23, DKK+23, GKSV23, STW24}.
There are efficient algorithms for testable learning in cases where distribution-free learning cannot be done efficiently.

\subsection{The Power of Membership Queries in Agnostic Learning}

One might hope to testably learn more efficiently by strengthening the learner's access to the data distribution.
In the \emph{membership query} (MQ) model, we think of the data as being drawn from a distribution $\mcD_x$ over $\mcX$ and labeled by some unknown function $f:\mcX \to \zo$.
The learner gets i.i.d.\ samples from $\mcD_x$ and may also query any point $x \in \mcX$ and receive its label $f(x)$.

The work of \cite{Feldman09} shows that under the standard cryptographic assumption of one-way functions, membership queries can speed up agnostic
learning in the distribution-specific setting, but not in the distribution-free setting.
In the distribution-free setting, every concept class that can be agnostically learned with MQs can be agnostically learned with just random examples.
In contrast, in the uniform-distribution-specific setting, there exists a concept class that can be learned strictly more efficiently with MQs.

Separations exist for more ``natural'' concept classes under the stronger assumption that learning sparse parities with noise (LSPN) is hard.
For example, over the uniform distribution on $\zo^n$,
$k$-juntas can be learned in $\poly(n) \cdot 2^{O(k)}$ time with membership queries \cite{BL97, MOS04}.
On the other hand, there is a statistical query lower bound of $n^{\Omega(k)}$ \cite{BF02}, 
and if LSPN is hard then one cannot hope to do better than this bound with random examples.
Similarly, polynomial-size decision trees can be learned improperly in polynomial time \cite{KM91, GKK08} and properly in $n^{O(\log \log n)}$ time \cite{BLQT22} with membership queries, 
while there is an SQ lower bound of $n^{\Omega(\log n)}$ \cite{BFJKMR94}, and LSPN implies one cannot do better than this bound either.

\subsection{Limitations of Membership Queries in Testable Learning}
If membership queries do help in the distribution-specific setting but do not help
in the distribution-free setting,
one may then naturally wonder whether they ought to help in the testable setting as well. 
A \emph{testable learner with queries} (TL-Q) has both sample access to the unknown data distribution and membership query access to the unknown function,
and must satisfy the soundness and completeness guarantees of ordinary testable learning.

\begin{question}
\label{q:mq-speedup}
How much can membership queries speed up the task of testable learning? 
\end{question}

Our results show that membership queries are quite weak in the TL-Q setting. 
Particularly, whenever agnostic learning with random examples is hard --- as is believed to be the case for juntas and decision trees --- testable learning
is hard as well, even with queries.

\begin{theorem}[\Cref{cor:tlq-agnostic}, informal]
If a concept class $\mcC$ is agnostically testably learnable with queries in time $t$ over a distribution $\mcD$, then it is agnostically learnable with random examples in time $\poly(t)$ over $\mcD$ as well.
\end{theorem}

\begin{corollary}
If LSPN is hard, then no concept class containing $k$-parities as a subset 
can be agnostically testably learned in $n^{o(k)}$ time 
over the uniform distribution,
even with membership queries.
\end{corollary}
Furthermore, we show that SQ lower bounds rule out a large class of natural query-based learning algorithms. 
We define a class of ``statistical'' membership query (MQ-SQ) algorithms --- those that use membership queries only to sample from particular distributions over the input domain.
For example, algorithms that use MQs only to estimate influences or to make SQs over large subsets of $\zo^n$ are MQ-SQ algorithms (this includes the aforementioned uniform-distribution algorithms of \cite{KM91, GKK08, BLQT22}).
We show that such algorithms cannot be ``made testable'' with respect to the uniform distribution without introducing non-statistical use of membership queries, due to the SQ lower bounds.
\begin{theorem}[\Cref{thm:MQ-SQ-lower-bound-in-SQ-DIM}, informal]
If a concept class $\mcC$ is testably learnable in time $t$ over a distribution $\mcD$ by an MQ-SQ algorithm, then the SQ dimension of $\mcC$ with respect to $\mcD$ is at most $\poly(t)$.
\end{theorem}

\subsection{Technical Overview}

Our reductions are through the intermediate task of \emph{refutation}.
Refutation, as presented in 
\cite{Vadhan17, KL18}, is the problem of distinguishing examples 
correlated with some function in the concept class from examples labeled uniformly at random.
The work of \cite{Vadhan17} shows that in the distribution-free setting, refutation and realizable learning are polynomially equivalent, 
and the work of \cite{KL18} shows an 
analogous statement in the distribution-specific agnostic setting.
By giving an efficient reduction from distribution-specific refutation (without queries) to testable learning (with queries),
we show that distribution-specific agnostic learning reduces to TL-Q as well.

As a warm-up, consider a special case of refutation where the labels are promised to either be completely random or exactly match some function in the class $\mcC$.
Let the distribution be uniform over $\zo^n$.
Assume for simplicity that all functions in $\mcC$ are balanced, i.e., $\Ex{x \sim \zo^n}{f(x)}= 1/2$.

\begin{definition}[Exact refutation over the uniform distribution, informal]
An exact refutation algorithm for a concept class $\mcC$ takes an $m$-tuple $\{(x_1,y_1),\ldots,(x_m,y_m)\}$ of examples where the $x$'s are drawn uniformly at random from $\zo^n$.
It outputs either $\noise$ or $\structure$ with the following guarantees:
\begin{itemize}
\item \emph{Completeness:} If the examples are consistent with some $g \in \mcC$,
then
\[\Prx[\mcA\text{ outputs }\mathsf{structure}] \ge 2/3.\]
\item \emph{Soundness:} If the $y_i$'s are drawn i.i.d. from $\Bern(1/2)$, then
\[\Prx[\mcA\text{ outputs }\mathsf{noise}] \ge 2/3.\]
\end{itemize}
\end{definition}

Suppose we want to implement exact refutation using a TL-Q algorithm.
If we had any agnostic learning algorithm that did not require queries, the task 
would be trivially easy: split $\{(x_1,y_1),\ldots,(x_m,y_m)\}$ into training and
test sets,
run the learner on the training set, estimate the error of the returned hypothesis on the test set, and output $\structure$ if the test error is, say, $\le 1/10$.
If we are in the $\structure$ case, the error will be $\le \eps$, and if we are in the $\noise$ case, with high probability the error will be close to 1/2.

Instead we have to answer queries, so we will answer them randomly. 
Specifically, we will draw a random function $f: \zo^n \to \zo$, and whenever the TL-Q algorithm
wants to make a query, we will answer according to the $f$ we chose.
We will filter both the training and test sets to just those points where $y = f(x)$.
This means that we essentially sample from a domain that is uniform over the portion of $\zo^n$ where $y(x)$ agrees with $f(x)$.

As before, if the TL-Q algorithm produces a hypothesis, we will output $\structure$ if the error is $\le 1/10$ and $\noise$ if the error is greater.
But since TL-Q can also reject the instance and output $\bot$, if it does so,
we will output $\structure$.

Notice that in the noise case, each $x_i$ is filtered out independently with probability 1/2; therefore the distribution of samples is uniform.
By completeness of the TL-Q algorithm, it must then output a hypothesis, and with
high probability the error will be close to 1/2.
In the structure case, however, the distribution of $x_i$'s may be far from uniform, in which case the TL-Q algorithm may output $\bot$. 
It may also output a hypothesis --- but by soundness, the hypothesis must have 
error $\le \eps$, since the samples come from a distribution such that $y= f(x)$
for every $x$ in its support.

Our reduction from refutation to TL-Q is basically a generalization of this strategy, adapted to handle unbalanced functions and accept functions that are close to,
but not exactly, in $\mcC$.

\subsubsection{An SQ-Preserving Reduction}
We observe that some membership query algorithms, such as those of \cite{KM91, GKK08, BLQT22},
use membership queries only to estimate statistical properties of the unknown function. 
We roughly categorize MQ-SQ queries as follows (formalized in \Cref{def:MQ-SQ-oracle}):
\begin{itemize}
\item Standard SQs: queries of the form $\Ex{x \sim \calD}{\phi(x)}$ or $\Ex{x \sim \calD}{\phi(x)f(x)}$ for a test function $\phi$. These don't require membership queries to implement.
\item Pair SQs: queries of the form $\Ex{x \sim \calD}{\phi(x)f(x)f(\pi(x))}$, where $\pi$ is a permutation of the domain without fixed points. 
This generalizes influence estimation.
\item Customized distribution SQs: Any of the above queries, where the expectation is taken over a specific (and sufficiently spread-out) distribution $\mcD^\star$ instead of the unknown distribution $\mcD$.
This generalizes making SQs over restrictions of $\zo^n$.
\end{itemize}

Our goal is to simulate an MQ-SQ testable learner by making only SQs to the unknown distribution $\Dref$ (over $\calX \times \zo$) in the refutation instance.
As in the non-SQ setting, we choose a random function to be the target function and answer queries according to that random function.

For example, the customized-distribution query $\Ex{x \sim \mcD^\star}{f(x)\phi(x)}$ is easy to simulate with just one SQ to $\Dref$: 
simply estimate the mean $p \coloneqq \Ex{(x,y) \sim \Dref}{y}$, and answer the MQ-SQ with $p \cdot \Ex{x \sim \mcD^\star}{\phi(x)}$ (sampling from $\mcD^\star$ requires neither samples nor membership queries, since it's the customized distribution).
The value of this MQ-SQ concentrates around this estimate as each $f(x)$ is an independent random variable with mean $p$.

Pair SQs are handled similarly, though in this case the random variables $f(x)f(\pi(x))$ are not independent.
However, since the dependence graph of these variables decomposes into cycles,
we can partition the graph into large independent sets and prove concentration of the variables within the independent sets.
Thus, for example, the MQ-SQ $\Ex{x \sim \mcD}{\phi(x)f(x)f(\pi(x))}$ can be answered with $p \cdot \Ex{(x, y) \sim \Dref}{\phi(x)y}$, which is an SQ to $\Dref$.

\subsection{Related Work and Discussion}
\paragraph{The power of membership queries.} It is well known that in the realizable setting, PAC learners with membership queries are strictly stronger than PAC learners without them under standard cryptographic assumptions \cite{ELSW07, FS09}.
The work of \cite{Feldman09} establishes an equivalence between PAC with random examples and PAC with membership queries in the distribution-free agnostic setting, and a separation in the distribution-specific agnostic setting.

\paragraph{Testable learning and friends.} 
There are many papers that
address the computational and sample complexities of testably learning various natural concept classes;
these works are in the standard agnostic testable learning model.
Some examples, but certainly not all, are the works of \cite{RV23, GKK23, DKK+23, GKSV23, STW24}.
Some works addressing related problems include \cite{KSV24a, KSV24b, MRS25}.

The work of \cite{GKK23} characterizes the sample complexity of testable learning by the Rademacher complexity, which is especially relevant to this work in light of the result of \cite{KL18}, which establishes refutation complexity as an analogue of Rademacher complexity for the computationally bounded setting.
\paragraph{Learning and refutation.}
A connection between learning and refutation
was first introduced in \cite{DLS14}
as a means of proving computational lower bounds for learning based on the assumption that refuting random CSPs is hard,
and other works including those of \cite{DS16, Daniely16} use this method to give conditional lower bounds for various learning problems.
Of particular relevance to this work are \cite{Vadhan17, KL18}, which give polynomial equivalences between PAC learning and refutation.

\subsubsection{Directions for Future Work}
While our work reduces ordinary sample-based PAC learning to query-based testable learning,
we do not resolve the strongest, most natural question on the power of membership queries: whether sample-based \emph{testable} learning reduces to query-based testable learning.
This would be a strictly stronger lower bound for TL-Q than anything obtainable through refutation,
as there are function classes for which refutation is known to be easier than testable learning with samples --- for example, the class of monotone functions.
It is proven in \cite{RV23} that testably learning monotone functions on the uniform distribution requires $2^{\Omega(n)}$ samples.
On the other hand, agnostic learning (and therefore refutation) can be done in $2^{\tilde{O}(\sqrt{n})}$ time and samples.

We leave this possible stronger lower bound as an open question for future work.

\begin{conjecture}
If a concept class $\mcC$ is agnostically testably learnable with queries in time $t$ over a distribution $\mcD$, then it is agnostically testably learnable with samples in time $\poly(t)$ over $\mcD$ as well.
\end{conjecture}

We also remark that in the semi-agnostic setting, our method relates semi-agnostic TL-Q to \emph{weak} agnostic learning.
It is an open question for future work to resolve the connection between semi-agnostic TL-Q and semi-agnostic learning as well.

\section{Preliminaries}
\subsection{Distances and Errors}
\begin{definition}[Distance of functions and distance to a concept class]
Relative to a distribution $\mcD$ over $\mcX \times \zo$ with $\mcX$-marginal $\mcD_x$, we denote
\[\dist_{\mcD_x}(f,g) = \Prx_{x \sim \mcD_x}[f(x) \ne g(x)]\]
\[\err_{\mcD}(f) = \Prx_{(x,y) \sim \mcD}[f(x) \ne y].\]
We also use the following notation to denote the classification error of the most
accurate concept in a class, which we often refer to as $\opt$:
\[\dist_{\mcD_x}(f, \mcC) = \inf_{g \in \mcC}\dist_{\calD_x}(f, g).\]
\[\err_{\mcD}(\mcC) = \inf_{g \in \mcC}\err_{\calD}(g).\]

\end{definition}

\subsection{Refutation and Learning}
We state a definition of refutation similar to the definition presented in \cite{KL18}.
We have modified it to use classification error rather than correlation,
for ease of use in our $\zo$-labeled setting (\cite{KL18} uses $\bits$ labels).

\begin{definition}[$\eta$-refutation]
\label{def:unbiased-refutation}
Let $\mcC \subseteq \{f:\mcX \to \zo\}$ be a concept class over a finite input domain $\mcX$, and let $\mcF$ be a
family of distributions
on $\mcX$. An $\eta$-refutation algorithm $\mcA$ for $\mcC$ on $\mcF$ with $m$ 
samples is an algorithm that takes an $m$-tuple of labeled examples 
$\{(x_1,y_1),\ldots,(x_m,y_m)\}$ and outputs either $\mathsf{noise}$ or $\mathsf{structure}$.
If the examples are i.i.d. from a distribution $\mcD$ over $\mcX \times \zo$ such that the marginal on $\mcX$ is some $\mcD_x \in \mcF$, then the following guarantees hold:
\begin{itemize}
\item \emph{Completeness:} If there exists $g \in \mcC$ such that $\err_{\calD}(g) \le \eta$,
then
\[\Prx_{\substack{\{(x_i,y_i)\} \sim \mcD \\ \text{internal randomness of }\mcA}}[\mcA\text{ outputs }\mathsf{structure}] \ge 2/3.\]
\item \emph{Soundness:} If the $y_i$'s are drawn i.i.d.\ from $\Bern(1/2)$, then
\[\Prx_{\substack{\{(x_i,y_i)\} \sim \mcD \\ \text{internal randomness of }\mcA}}[\mcA\text{ outputs }\mathsf{noise}] \ge 2/3.\]
\end{itemize}
\end{definition}

We also define a similar but stronger task:

\begin{definition}[Biased $(\alpha,\eta)$-refutation]
\label{def:refutation}
A biased-$(\alpha, \eta)$-refutation algorithm is as above except the soundness condition is the following:
\begin{itemize}
\item \emph{Soundness:} For all $p \in [\alpha, 1-\alpha]$, if the $y_i$'s are drawn i.i.d. from $\Bern(p)$, then
\[\Prx_{\substack{\{(x_i,y_i)\} \sim \mcD \\ \text{internal randomness of }\mcA}}[\mcA\text{ outputs }\mathsf{noise}] \ge 2/3.\]
\end{itemize}
\end{definition}

Here we state definitions and facts used in \cite{KL18}'s reduction from refutation to agnostic learning.
Again we modify these statements to use classification error rather than correlation.

\begin{definition}[Weak agnostic learning]
A $(\gamma, \alpha)$-weak agnostic learner for the concept class $\mcC$ over the distribution $\mcD$ 
outputs a hypothesis $h$ satisfying the following:
\[\err_\mcD(h) \le \frac{1+\alpha-\gamma}{2} + \gamma \err_\mcD(\mcC).\]
\end{definition}

\begin{lemma}[Learning by refutation: Lemma 6 of \cite{KL18}]
\label{lem:KL18-reduction}
Suppose there is an $\eta$-refutation algorithm for the class $\mcC$ over distribution $\mcD$ running in $T(n)$
time with $m$ samples. 
Then there is an algorithm that runs in $T(n) \cdot \frac{m^2}{\eps^2}$ and uses $O(\frac{m^3}{\eps^2})$ samples to
agnostically learn $\mcC$ on $\mcD$ with excess error $1 - 2\eta + \eps$.
\end{lemma}

\subsection{Testable Learning with Queries}

We now define semi-agnostic \emph{testable learning with queries}.

\begin{definition}[Testable learning with queries, or TL-Q]
A concept class $\mcC$ over the input domain $\mcX$ is $(c, \eps,\delta)$-PAC-testably-learnable with $q$ queries, $m$ samples, and $t$ time, 
on a set $\mcF$ of distributions over $\mcX$, if 
there is a $t$-time algorithm $\mcA$ that takes $m$ samples from an unknown distribution $\mcD \in \mcF$ and $q$ membership queries to an unknown function $f^\star$, and outputs $h \in \mcC \cup \{\bot\}$ such that:
\begin{itemize}
\item Soundness: If $h \ne \bot$, then 
\[\Pr[\dist_{\mcD}(h,f^\star) \ge c \cdot \opt + \eps] \le \delta.\]
\item Completeness: If $\mcD \in \mcF$, then 
\[\Pr[h= \bot] \le \delta.\]
\end{itemize}
\end{definition}

\section{Refutation, Learning, and Testable Learning with Queries}
\subsection{A General Reduction from Refutation to TL-Q}
\label{sec:reduction}
In this section, we show that if a class is efficiently testably-learnable, then it is efficiently refutable as well, 
with polynomial dependence on the sample, time, and query complexity of the testable learner.

In fact, our \Cref{alg:refutation} solves the harder problem of biased refutation (\Cref{def:refutation}), though the distinction between the two will not be relevant until \Cref{sec:juntas}.
A biased refutation algorithm can always be used to solve the standard unbiased refutation problem, as the soundness guarantee for biased refutation must always hold when $p=1/2$.
To avoid confusion, we let $\Dref$ denote the distribution of labeled pairs $(x,y) \in \mcX \times \zo$ in the refutation instance,
while $\mcD$ denotes the unknown marginal distribution in a TL-Q instance.

\begin{algorithm}[h]
\begin{algorithmic}[1]
\caption{$\textsc{BiasedRefutation}(\mathsf{samples}, \eta, \eps, m,q,c)$:
\label{alg:refutation}}
\State \textbf{Input:} sample set $\mathsf{samples}$ of size $m'$ drawn from the refutation distribution $\Dref$, gap parameters $\eta$ and $\eps$, TL-Q parameters $m, q, c$
\State \textbf{Output:} $\noise$, $\structure$, or an error
\State
\State Use the first $C_1 \cdot 1/\eps^2$ samples to estimate 
$p \coloneqq \Ex{(x,y) \sim \Dref}{y}$.
If the estimated value $\hat{p} < 2\eps$ or $\hat{p} > 1-2\eps$, $\Return$ $\structure$.
\State Reserve the next $C_2 \cdot \Paren{\frac{m+1/\eps^2}{\eps} + q}$ samples to be used only for their labels (whenever a draw from $\Bern(p)$ is required, use a new label from this set). 
\State Initialize an empty set $S$ and an empty truth table $f$.
\For{each remaining example $(x_i,y_i)$}
\State Draw $b \sim \Bern(p)$ and store $f(x_i) = b$.
\State If $y_i = f(x_i) = 0$, add $x_i$ to $S$ with probability $p$.
\State If $y_i = f(x_i) = 1$, add $x_i$ to $S$ with probability $1-p$.
\State If $y_i \ne f(x_i)$, continue without adding $x_i$ to $S$.
\EndFor
\State If $|S| < m + C_3/\eps^2$, $\Return$ an error. Let the first $m$ examples be $S_{\mathsf{train}}$ and the remaining be $S_{\mathsf{test}}$.
\State Call the TL-Q algorithm $\mcA$ with the first $m$ members of $S$. Whenever $\mcA$ makes a query to some point $x$, if $f(x)$ is in the table, answer with $f(x)$; otherwise draw $b \sim \Bern(p)$, store $f(x)=b$, then answer $f(x)$.
\State If $\mcA$ returned $\bot$, $\Return$ $\structure$.
\State If $S_{\mathsf{test}}$ contains any duplicate elements, elements in $S_{\mathsf{train}}$, or elements queried by $\mcA$, $\Return$ an error.
\State If $\mcA$ returned a hypothesis $h$, evaluate
\[\wh{\mathsf{err}}(h) = \Prx_{x \sim S_{\mathsf{test}}}[h(x) \ne f(x)].\]
\State If $\wh{\mathsf{err}}(h) > c\eta + 3\eps$, $\Return$ $\noise$. Otherwise $\Return$ $\structure$.
\end{algorithmic}
\end{algorithm}

The main result of this section is the following:

\begin{theorem}
\label{thm:tlq-refutation}
Let $\mcC$ be $(c, \eps, \tfrac{1}{10})$ PAC-testably-learnable with $m$ samples, $q$ queries, and $t$ time, on a distribution family $\mcF$ satisfying 
\[m + q/\eps^2 \ll \frac{1}{\sup_{\mcD_x \in \mcF}(\norm{\mcD_x}_2)}.\]
Then for any $\eps$ satisfying $\eps^2 \ge ck \cdot \sup_{\mcD_x \in \mcF} (\norm{\mcD_x}_2)$ for sufficiently large constant $k$, and any $\eta < \frac{1/2 - 4\eps}{c}$, $\mcC$ is $(c\eta + 4\eps, \eta)$-refutable over all members of $\mcF$ with $m'$ samples and $t'$ time, where
\[m' = O\Paren{\frac{m + 1/\eps^2}{\eps} + q}\]
\[t' = O(m' + t).\]
\end{theorem}
\Cref{alg:refutation} essentially draws a random function $f$ of bias $p$, where $p$ is the mean of the labels in the refutation distribution $\Dref$, and filters the samples to just pairs where $y = f(x)$.
The drawing of $f$ is ``lazy:'' to draw $f$ and answer queries to it, it suffices to draw each value of $f(x)$ from $\Bern(p)$ the first time we need to know $f(x)$, then store $(x,f(x))$ in a table for consistency.
We implement the random coin by reading a label,
as the labels are distributed as $\Bern(p)$.
\subsubsection{Properties of the Filtered Sample Distribution}
\label{sec:filtered-sample}
Before proving the theorem, it will be useful to prove the following supporting claims about the distribution that the sets $S_{\mathsf{train}}$ and $S_{\mathsf{test}}$ are drawn from.
This distribution $\mcD$ --- which is the unknown distribution that the TL-Q instance is running on --- depends on the random function $f$.
Formally, the PMF of $\mcD$ is the following, and one may observe that the construction of $S$ in lines 7-11 of \Cref{alg:refutation} produces samples from this distribution:
\begin{definition}[Filtered sample distribution]
Let $\Dref$ be a distribution over $\mcX \times \zo$ with $\mcX$-marginal $\mcD_x$ and let $f: \mcX \to \zo$ be a $p$-biased random function. 
Let the function $y(x)$ be defined as $\Ex{(x',y') \sim \Dref}{y'\mid x'=x}$.
Then the filtered sample distribution $\mcD$ is defined by the following PMF:
\[\mcD(x) = \frac {\mcD_x(x)}{Z} \cdot (p(1-y(x))(1-f(x)) + (1-p)y(x)f(x)),\]
where $Z$ is the normalization factor
\[Z = \Ex{x \sim \mcD_x}{(1-p)y(x)f(x) + p(1-y(x))(1-f(x))}.\]
\end{definition}

We first show that $Z$ concentrates around $p(1-p)$.

\begin{lemma}\label{lemma:Z-concentration}
    For every $\delta \ge 0$, it holds with probability at least $1 - 2\exp\left(-\Omega\left(\frac{\delta^2 \cdot p^2(1-p)^2}{\|\calD_x\|_2^2}\right)\right)$ over the randomness of $f$ that
    \[
        \left|Z - p(1-p)\right| \le \delta \cdot p(1-p).
    \]
\end{lemma}

\begin{proof}
    Note that for any $x \in \calX$, regardless of the values of $p, y(x) \in [0, 1]$, and $f(x) \in \zo$,
    \[
        (1 - p) \cdot y(x) \cdot f(x) + p \cdot (1 - y(x)) \cdot (1 - f(x))
    \]
    is always between $0$ and $1$. Thus, over the randomness of $f$,
    \[
        Z = \sum_{x \in \calX}\calD_x(x)\left[(1 - p) \cdot y(x) \cdot f(x) + p \cdot (1 - y(x)) \cdot (1 - f(x))\right]
    \]
    is a sum of $|\calX|$ independent random variables, and the summand that corresponds to each $x \in \calX$ lies in $[0, \calD_x(x)]$. Furthermore, since $\Ex{}{f(x)} = p$ holds for every $x \in \calX$, we have
    \[
        \Ex{}{Z}
    =   \sum_{x \in \calX}\calD_x(x)\left[(1 - p) \cdot y(x) \cdot p + p \cdot (1 - y(x)) \cdot (1 - p)\right]
    =   p(1-p).
    \]
    Therefore, the lemma follows from Hoeffding's inequality.
\end{proof}

We now show that the distance of a function to $f$, with respect to $\mcD$, concentrates around the error of that function on $\Dref$.
\begin{claim}
\label{clm:error-blowup-better}
Let $g: \mcX \to \zo$ be an arbitrary function. 
For any $\delta$, with probability at least $1 - \exp \Paren{-\Omega \Paren{\frac{(\delta p)^2(1-p)^2}{\norm{\mcD_x}_2^2}}}$ over the randomness of $f$, we have
\[\Prx_{x\sim \mcD}[g(x) \ne f(x)] \le \Prx_{(x,y)\sim \Dref}[g(x) \ne y] + \delta.\]
\end{claim}

\begin{proof}
First, we note that 
\[\Prx_{(x,y) \sim \Dref}[g(x) \ne y] = \Ex{x \sim \mcD_x}{g(x)(1-y(x)) + y(x)(1-g(x))}.\]
For each fixed $f$, we have:
\begin{align*}
\Prx_{x \sim \mcD}[g(x) \ne f(x)]&= \frac 1Z \Ex{x \sim \mcD_x}{(p(1-y(x))(1-f(x)) + (1-p)y(x)f(x)) \cdot \Ind[g(x) \ne f(x)]}\\
&=\frac 1Z \Ex{x \sim \mcD_x}{p(1-y(x))(1-f(x))g(x) + (1-p)y(x)f(x)(1-g(x))}.
\end{align*}
By \Cref{lemma:Z-concentration} we have that with probability at least $1 - \exp \Paren{-\Omega \Paren{\frac{(\delta p)^2(1-p)^2}{\norm{\mcD_x}_2^2}}}$ over the randomness of $f$, $Z \ge (1- \tfrac{\delta}{4})p(1-p)$.
Now we examine the random variable
\[W \coloneqq \Ex{x \sim \mcD_x}{p(1-y(x))(1-f(x))g(x) + (1-p)y(x)f(x)(1-g(x))}.\]
In expectation over $f$, we have:
\begin{align*}
\Ex{f}{W} &=\Ex{x \sim \mcD_x}{p(1-y(x))(1-p)g(x) + (1-p)y(x)p(1-g(x))} \\
&= p(1-p) \Ex{x \sim \mcD_x}{(1-y(x))g(x) + y(x)(1-g(x))}\\
&= p(1-p) \Prx_{(x,y) \sim \Dref}[g(x) \ne y].
\end{align*}
Since $p(1-y(x))(1-f(x))g(x) + (1-p)y(x)f(x)(1-g(x))$ is always bounded between 0 and 1,
the random variable $W$
is a sum of $|\mcX|$ elements, each of which is bounded by $\mcD_x(x)$.
Therefore, by Hoeffding's inequality:
\begin{align*}
\Prx_f[W > p(1-p) (\Prx_{(x,y) \sim \Dref}[g(x) \ne y] + \lfrac{\delta}{4})] \le \exp\Paren{-\Omega \Paren{\frac{(\delta p)^2(1-p)^2}{\norm{\mcD_x}_2^2}}}.
\end{align*}

Whenever $W \le p(1-p) (\Prx_{(x,y) \sim \Dref}[g(x) \ne y] + \lfrac{\delta}{4})$ and $Z \ge p(1-p)(1 - \lfrac{\delta}{4})$, we have 
\[\frac WZ \le \frac{p(1-p)(\Prx[g(x) \ne y] + \lfrac{\delta}{4})}{p(1-p)(1 - \lfrac{\delta}{4})} \le \frac{\Prx[g(x) \ne y] + \lfrac{\delta}{4}}{1 - \lfrac{\delta}{4}} \le \Prx[g(x) \ne y] + \delta.\]
The claim follows by a union bound over the tail events of $W$ and $Z$.
\end{proof}

\subsubsection{Proof of \Cref{thm:tlq-refutation}}
With the above properties in hand, we will now prove the main theorem of this section.

\begin{proof}[Proof of \Cref{thm:tlq-refutation}]
Let $\mcA$ be the $(c,\eps, 1/10)$-testable learner for $\mcC$. We will show that \Cref{alg:refutation} is a
$(c\eta + 4\eps, \eta)$ refutation algorithm over any $\Dref$ with $\mcX$-marginal $\mcD_x$ such that $\mcD_x \in \mcF$.
Since the samples come from a refutation instance, one of the following must hold:
\begin{itemize}
		\item Structure: $\Prx_{(x,y)\sim \Dref}[g(x) \ne y] \le \eta$ for some $g \in \mcC$, or
		\item Noise: $y \sim \Bern(p)$ for some $p \in [c\eta + 4\eps, 1-c\eta - 4\eps]$.
\end{itemize}

We will refer to $\Ex{(x,y) \sim \Dref}{y}$ as $p$ regardless of whether we are in the structure case or the noise case.
We consider the following possibilities:
\begin{itemize}
\item \textbf{Structure:} In this case, there is some function $g \in \mcC$ such that $\Prx_{(x,y)\sim \Dref}[g(x) \ne y] \le \eta$.
By setting the constant $C_1$ large enough, we have by Hoeffding's inequality that $\Prx[|\hat{p} - p| > \eps] \le 1/100$.
With the remaining probability, if $p < \eps$ or $p > 1-\eps$ we would output $\structure$ after line 4,
so we will assume from here that $\eps \le p \le 1-\eps$.

We will denote by $\mcD$ the distribution over $\mcX$ from which our sample set $S$ is drawn, as discussed in \Cref{sec:filtered-sample}.
Since $\eps \le p \le 1 - \eps$, by our assumption that $\eps^2 \ge ck \cdot \sup_{\mcD_x \in \mcF} (\norm{\mcD_x}_2)$ we have
\[\frac{(\eps/c)^2 \cdot p^2(1-p)^2}{\norm{\mcD_x}_2^2} \ge \Omega\Paren{\frac{\eps^4}{c^2\norm{\mcD_x}_2^2}} \ge \Omega(k^2).\]

By \Cref{clm:error-blowup-better}, setting the constant $k$ to be large enough, we then have
\[\dist_{\mcD}(f, \mcC) \le \Prx_{x \sim \mcD}[f(x) \ne g(x)] + \eps/c \le \eta +\eps/c\]
with probability at least $1 - \exp (-\Omega(k^2)) \ge 99/100$
over the randomness of $f$.
When this happens, $\mcA$ has two possible sound behaviors: either output $\bot$, or output $h$ satisfying 
\[\dist_{\mcD}(h,f) \le c(\eta + \eps/c) + \eps \le c\eta + 2\eps.\]
By the TL-Q guarantee, it produces one of these sound behaviors with probability at least $9/10$.
By setting the constant $C_3$ large enough, by Hoeffding's inequality if $\mcA$ outputs a hypothesis $h \ne \bot$, it satisfies $\wh{\mathsf{err}}(h) \le \dist_{\mcD}(h,f) + \eps$ with probability at least 
\[1 - \exp(-2 \eps^2 |T|) \ge 99/100.\]
When this happens, we have $\wh{\mathsf{err}}(h) \le c\eta + 3 \eps$,
so we successfully output $\mathsf{structure}$.
\item \textbf{Noise:} In this case, the labels are drawn from $\Bern(p)$, and the elements are drawn from $\mcD$.
Since $y(x) = p$ for all $x$, we have 
\[\mcD(x) = \frac{\mcD_x(x) \cdot p(1-p)}{Z} \quad \text{and} \quad Z = \sum_{z \in \mcX} \mcD_x(x) \cdot p(1-p).\]
Thus, $\mcD = \mcD_x$, so
by completeness of $\mcA$, it must output $\bot$ with probability at most $1/10$. 
With the remaining probability, $\mcA$ outputs a hypothesis $h$; we will argue that with high probability
its error on $S_{\mathsf{test}}$ is at least $\min(p,1-p) - \eps > c\eta + 3\eps$.
We return an error if any element in the test set appears anywhere else,
so assuming this does not happen,
each label in the test set is a new independent draw from $\Bern(p)$.
Thus we have:
\begin{align*}
\Pr_{x \sim T}[h(x) \ne f(x)] &= \frac{1}{|T|} \sum_{x \in T:h(x) =0} \Bern(p) + \frac{1}{|T|} \sum_{x \in T:h(x)=1} \Bern(1-p) \\
&\ge \frac{1}{|T|} \mathsf{Binomial}(\min(p,1-p), |T|).
\end{align*}
By a Hoeffding bound it follows that $\wh{\mathsf{err}}(h) \ge \min(p,1-p) - \eps$
with probability at least 
\begin{align*}
1 - 2\exp(-2\eps^2 |T|) &\ge 99/100.
\end{align*}
Thus, we successfully output $\mathsf{noise}$ with high probability.
\end{itemize}
In both cases the refutation algorithm succeeds with probability $\ge 4/5$ conditioned on not returning an error due to either insufficient samples, duplicate samples, or overlap between the query set and the test set.

We will set the sample complexity so that the probability of insufficient samples is small. 
Let $B$ be the number of samples reserved for drawing from $\Bern(p)$ or estimating $\hat{p}$.
Each of the remaining $m' - B$ samples is included in $S$ with probability $p(1-p)$.
Thus we will set $m'-B \ge \frac{100}{(c\eta + 4\eps)(1-c\eta - 4\eps)} \cdot(m + C_3/\eps^2)$.
Then we have by a Chernoff bound,
\[\Pr\left [|S| < m + \frac{C_3}{\eps^2} \right] \le \exp\Paren{-\frac{(99/100)^2}{2} \cdot p(1-p)(m' - B)} \le 1/100. \]
To set $B$, we need 2 draws for each of the samples and 1 draw for each membership query; by setting $C_2$ large enough this requirement is satisfied.
Thus the total sample complexity is 
\[m' \coloneqq O \Paren{\frac{m + 1/\eps^2}{\eps} + q}\]
and the total time complexity is $O(m' + t)$.

Finally, we will bound the probability of duplicate samples and overlap with the query set.
By the assumption that $m + 1/\eps^2 \ll 1/\norm{\mcD_x}_2$ and the fact that each pair of samples collides with probability $\norm{\mcD_x}_2^2$, it follows from a union bound over the pairs of samples in $S$ that w.h.p. there are no duplicates in $S_{\mathsf{train}} \cup S_{\mathsf{test}}$.
Furthermore, by the assumption that $q/\eps^2 \ll 1/\norm{\mcD_x}_2 \le 1/\norm{\mcD_x}_\infty$, it follows that every set of size $q$ has distributional mass $\ll \eps^2$.
Thus, with high probability none of the $C_3/\eps^2$ elements in the test set appear in the query set. 

Union bounding over all the failure probabilities in each case,
the total success probability remains at least 2/3.
\end{proof}
\subsection{TL-Q Implies Sample-Based Learnability}

A corollary of \Cref{thm:tlq-refutation} is the fact that TL-Q implies efficient sample-based, 
distribution-specific agnostic learning.
\begin{corollary}[Testable learning with queries implies learning with samples]
\label{cor:tlq-agnostic}
Let $\mcC$ be $(c, \eps, \tfrac{1}{10})$ PAC-testably-learnable with $m$ samples, $q$ queries, and $t$ time, on a distribution family $\mcF$ satisfying 
\[m + q/\eps^2 \ll \frac{1}{\sup_{\mcD_x \in \mcF}(\norm{\mcD_x}_2)}.\]

Then for any $\eps$ satisfying $\eps^2 \ge ck \cdot \sup_{\mcD_x \in \mcF} (\norm{\mcD_x}_2)$ for sufficiently large constant $k$, there is an agnostic learner for $\mcC$ over $\mcF$ with 
excess error $1 - 1/c + O(\eps)$.
In particular, when $c=1$, i.e. $\mcC$ is fully agnostically learnable in TL-Q, $\mcC$ is fully agnostically learnable with samples.

The sample complexity of the learner is
\[O((m')^3/\eps^2)) \quad \text{where} \quad m' = O\Paren{\frac{m + 1/\eps^2}{\eps} + q}\]
and the time complexity is
\[O\Paren{\frac{(m')^2(m' + t)}{\eps^2}}.\]
\end{corollary}

\begin{proof}
By \Cref{thm:tlq-refutation}, there is a $(c\eta + 4\eps, \eta)$-refutation algorithm for any $\eta < \frac{1/2-4\eps}{c}$.
Observe that biased $(\alpha, \eta)$-refutation is at least as strong as $\eta$-refutation: the $(\alpha, \eta)$-refutation algorithm can be used to solve $\eta$-refutation, as $p=1/2$ is always in the range $[\alpha, 1-\alpha]$.
By \Cref{lem:KL18-reduction}, this gives an agnostic learner with excess error $1 - 2 \cdot \frac{1/2-4\eps}{c} + \eps = 1 - 1/c + O(\eps)$.
The time and sample bounds are obtained by combining the bounds in \Cref{lem:KL18-reduction} with those in \Cref{thm:tlq-refutation}.
\end{proof}

\subsection{Realizably Learning Juntas via Exact Refutation}
\label{sec:juntas}
In the above subsections, we gave a general reduction from TL-Q to agnostic learning,
citing as a black box the learning-by-refutation lemma of \cite{KL18}, \Cref{lem:KL18-reduction}.
This lemma yields learners whose excess error
depends on the refutation gap parameter $\eta$, which in our case must necessarily be smaller than $1/2c$, as one cannot hope to 
distinguish a function that is $(1/2c)$-close
to the concept class from a random function using a $c$-semi-agnostic learner.
Thus the performance of this learner quickly degrades with $c$, becoming trivial when $c=2$.

In this section, for the class of sparse juntas over the uniform distribution on $\zo^n$, we give a realizable learner from an algorithm that solves the easier task of \emph{exact} refutation ($\eta = 0$).
Exact refutation reduces to $c$-semi-agnostic TL-Q for any value of $c$, via \Cref{thm:tlq-refutation}.
Thus we show that for juntas, even for large values of $c$, semi-agnostic TL-Q
is as hard as learning with samples.
\begin{lemma}
\label{lem:juntas}
Let $\eps \gg 2^{-n/4}$ and $m + q/\eps^2 \ll 2^{n/2}$.
For any constant $c$, if the class of $k$-juntas is $(c, \eps, \tfrac{1}{10})$-PAC-testably-learnable with $q$ queries, $m$ samples, and $t$ time over the uniform distribution on $\zo^n$,
then the class of $k$-juntas is agnostically learnable over the uniform distribution with excess error $O(\eps)$ and confidence $1-\delta$, with sample complexity 
\[m' = \Paren{\frac{m + 1/\eps^2}{\eps} + q} \cdot 2^{O(k)} \cdot n \log^2(n/\delta)\]
and time complexity 
\[t' = \Paren{\frac{m + 1/\eps^2}{\eps} + q + t} \cdot 2^{O(k)} \cdot n \log(n/\delta).\]
\end{lemma}

\subsubsection{Refutation to Feature Selection}
Suppose we have an $(\eps, 0)$-refutation algorithm for $k$-juntas.
We follow a simple reduction that appears in the COLT~2003 open problem of Blum~\cite{Blum03}, which uses weak learners
to identify relevant variables.
\begin{itemize}
    \item Draw labeled examples $\{(x_i, y_i)\}$ from the uniform distribution on $\zo^n$ labeled by the unknown $k$-junta $f^\star$.
    \item For each sample $x_i$, rerandomize the first $\ell$ bits to obtain $x'_i$.
    \item Run the refutation algorithm $\A$ on the resulting data $\{(x'_i, y_i)\}$.
\end{itemize}

Observe that when $\ell = n$, the samples are uniform random examples labeled by a random function of bias $p$, where $p$ is the bias of the original unknown function. 
Then, by soundness, $\A$ must output $\noise$ with high probability. Likewise, when $\ell = 0$, the samples are labeled by a $k$-junta, so $\A$ must output $\structure$ with high probability. 
Therefore, there must exist some $\ell^\star \in [n]$ such that the behavior of the algorithm is significantly different when $\ell = \ell^\star$ vs.\ $\ell = \ell^\star - 1$. 
We will show that for this to happen, $x_{\ell^\star}$ must be a relevant variable in $f^\star$. (Intuitively, masking one more irrelevant variable does not change the distribution of $(x'_i, y_i)$.)

\begin{definition}[Prefix-randomization]
For $x \in \zo^n$, we will denote by $x^{> \ell}$ the random string obtained by replacing each $x_i$ with a uniform random bit for each $i \le \ell$ and preserving $x_i$ for each $i > \ell$.
For a function $f:\zo^n \to \zo$, 
we will denote by $f^{> \ell}$ the stochastic function obtained by rerandomizing the first $\ell$ bits of the input: $f^{>\ell}(x) = f(x^{>\ell})$.
\end{definition}

The following useful statement holds simply from 
the fact that if the $\ell_{th}$ variable is not relevant,
then $f^{>\ell}$ and $f^{>\ell-1}$ are the same stochastic function;
i.e. $\Pr[f^{>\ell}(x)=1] = \Pr[f^{>\ell-1}(x)=1]$ for every $x$.
\begin{fact}
\label{fact:irrelevant}
Let $f$ be a $k$-junta and $\A$ be a refutation algorithm.
If the $\ell_{th}$ variable is not a relevant variable in $f$,

\[\Pr[\A\text{ outputs }\structure\text{ on }f^{>\ell}] = \Pr[\A\text{ outputs }\structure\text{ on }f^{>\ell-1}],\]
and likewise for $\noise$ and errors,
where the probability is taken over the random examples and the internal randomness of $\mcA$.
\end{fact}

Now we implement the feature selection algorithm described in \cite{Blum03}.
\begin{claim}
\label{clm:relevant}
If the class of $k$-juntas is $(\eps, 0)$-refutable over the uniform distribution with sample complexity $m$ and time complexity $t$,
then there is an algorithm that, for any $k$-junta with mean $\in [\eps,1-\eps]$, outputs a relevant feature with probability at least $1 - \delta$.
It has sample complexity 
\[m' = O(k^2 m n \log(n/\delta)).\]
and time complexity
\[t' = O(k^2 t n \log(n/\delta)).\]
\end{claim}
\begin{proof}
Let $f$ be a $k$-junta with mean at least $\eps$ and
let $\A$ be an $(\eps, 0)$-refutation algorithm for $k$-juntas.
Since $f^{>0}$ is a $k$-junta, by soundness of $\A$, we have 
\[\Pr[\A\text{ outputs }\structure\text{ on }f^{>0}] \ge 2/3.\]
Since $f^{>n}$ is a random function of bias $\in [\eps, 1-\eps]$, by completeness of $\A$ we have 
\[\Pr[\A\text{ outputs }\structure\text{ on }f^{>n}] \le 1/3.\]
By \Cref{fact:irrelevant}, there must then be some $i^\star$ for which 
\[\Pr[\A\text{ outputs }\structure\text{ on }f^{>i^\star}] - \Pr[\A\text{ outputs }\structure\text{ on }f^{>i^\star -1}] \ge 1/3k,\]
and this $i^\star$ must be a relevant variable.
We will estimate $\Pr[\A\text{ outputs }\structure\text{ on }f^{>i}]$ for every $i \in [n]$ to additive error $1/6k$ and confidence $1 - \delta/n$, and output such an index $i^\star$.

To estimate this probability to the desired accuracy and confidence, by Hoeffding's inequality, $O(k^2 \cdot \log (n/\delta))$ independent
runs of $\A$ suffice for each estimate.
This gives the desired running time and sample complexity bounds.
\end{proof}

\subsubsection{Proof of \Cref{lem:juntas}}

Before describing the junta learning algorithm, we introduce the following notation:
\begin{definition}[Restriction]
For a function $f$ and a node $v$ in a decision tree, we denote by $f_v:\bits^{n - \mathrm{depth}(v)}$ the restriction of $f$ to the subset $\{x \in \zo^n: x\text{ reaches }v\}$.
\end{definition}

Now we complete the proof that semi-agnostic TL-Q for $k$-juntas implies agnostic learning.

\begin{proof}[Proof of \Cref{lem:juntas}]
By \Cref{thm:tlq-refutation} and the fact that $\norm{\mcD_x}_2^2 = 2^{-n}$ when $\mcD_x$ is uniform over $\zo^n$, if $k$-juntas are testably learnable with queries then there exists an $(4\eps, 0)$-refutation algorithm with sample complexity 
\[m'' = O\Paren{\frac{m + 1/\eps^2}{\eps} + q}\]
and time complexity 
\[t'' = O\Paren{\frac{m + 1/\eps^2}{\eps} + q + t}.\]

Let $f$ be the unknown $k$-junta; we will build a $4\eps$-approximate decision tree for $f$,
beginning with the empty tree.
By \Cref{clm:relevant}, there exists an algorithm that outputs a relevant variable when given samples from a $k$-junta with mean $\in [4\eps, 1-4\eps]$.
While there exists a node $v$ such that the restriction $f_v$ is not $4\eps$-close to constant,
we will use this algorithm to find a variable to place at $v$.

By the promise that $f$ is a $k$-junta, the depth of this tree will be at most $k$.
Thus there are at most $2^{k+1}$ calls to the algorithm of \Cref{clm:relevant}.
So to union bound over all these calls, we will set $\delta' = \delta \cdot 2^{-(k+2)}$ in each call to the feature selection algorithm.
Since $\log(n/\delta') = O(k + \log(n/\delta))$, this gives the desired time complexity of 
\[t' = t'' \cdot 2^{O(k)} n \log (n/\delta).\]

We will also draw enough samples that with probability at least $1 - \delta/2$,
every restriction gets at least 
\[m'' \cdot 2^{O(k)} \cdot n \log(n/\delta)\]
samples consistent with it, which is the sample complexity of the feature selection algorithm.
Since the distribution is uniform over $\zo^n$,
the number of samples consistent with a depth-$k$ restriction is distributed as $\mathsf{Binomial}(2^{-k}, m')$.
The claimed sample complexity then follows by a Chernoff bound and $2^{k+1}$-wise union bound over the restrictions.  
\end{proof}

Since $k$-sparse parities are a subclass of $k$-juntas, we conclude that even semi-agnostic TL-Q cannot be done in $n^{o(k)}$ time,
under the assumption that LSPN is hard.
\begin{corollary}
If LSPN requires $n^{\Omega(k)}$ time, then for any constant $c$, $c$-semi-agnostic TL-Q for $k$-juntas over the uniform distribution requires $n^{\Omega(k)}$ time.
\end{corollary}

\section{MQ-SQ Lower Bounds}\label{sec:SQ-lower-bounds}
We introduce a class of ``MQ-SQ'' (\emph{membership-query-statistical-query}) algorithms that capture many existing learning algorithms that use membership queries. Then, we prove an SQ-analogue of the reduction in \Cref{sec:reduction}: an MQ-SQ testable learner for a class $\calC$ implies an SQ algorithm for refutation (\Cref{def:refutation}). This result, together with a reduction from SQ weak learning to SQ refutation, allows us to prove lower bounds against MQ-SQ algorithms for testably learning several fundamental concept classes, including parity functions, $k$-juntas, and decision trees.

\subsection{Five Types of MQ-SQs}
Let $\calX$ denote the instance space and $f: \calX \to \zo$ denote the target function in the TL-Q instance. Let $\calD \in \Delta(\calX)$ be the unknown marginal distribution over $\calX$. An MQ-SQ oracle for $(f, \calD)$ answers the following five types of queries up to a small error.

\begin{definition}[MQ-SQ Oracle]
\label{def:MQ-SQ-oracle}
    An MQ-SQ oracle with tolerance $\tau \ge 0$ answers the following five types of queries within an additive error of $\tau$, given any test function $\phi: \calX \to [0, 1]$, any distribution $\Dstar \in \Delta(\calX)$, and any permutation $\pi: \calX \to \calX$ without fixed points:
    \begin{itemize}
        \item Type I: $\Ex{x \sim \Dstar}{\phi(x)f(x)}$.
        \item Type II: $\Ex{x \sim \Dstar}{\phi(x)f(x)f(\pi(x))}$.
        \item Type III: $\Ex{x \sim \calD}{\phi(x)}$.
        \item Type IV: $\Ex{x \sim \calD}{\phi(x)f(x)}$.
        \item Type V: $\Ex{x \sim \calD}{\phi(x)f(x)f(\pi(x))}$.
    \end{itemize}
\end{definition}

For concreteness, consider the case that $\calX = \zo^n$ is the hypercube. By setting $\Dstar$ to the uniform distribution over $\zo^n$ (denoted by $\calU$) in Type~I queries, we recover the usual statistical query model for learning over $\calU$. Queries of Types II~and~V allow us to estimate the correlation between $f$ and a permuted version of $f$ (weighted by $\phi$) over both a generic customized distribution $\Dstar$ and the unknown marginal $\calD$, respectively. For instance, setting $\pi: x \mapsto x^{\oplus i}$ and $\Dstar = \calU$ allows us to estimate the influence of variable $x_i$ with respect to $f$.

When the distribution $\Dstar$ in a Type~I MQ-SQ is degenerate (i.e., with all its probability mass on a single point $x_0 \in \calX$), the MQ-SQ reduces to a usual membership query at $x_0$. Our reduction for MQ-SQ algorithms only works when no such queries are made. More concretely, the reduction requires that the squared $2$-norm of $\Dstar$, $\|\Dstar\|_2^2 \coloneqq \sum_{x \in \calX}[\Dstar(x)]^2$, is \emph{sufficiently small} in all queries that the algorithm makes. As we will see later, when many existing query-based learning algorithms are implemented using MQ-SQs, $\Dstar$ is usually uniform over a size-$2^{\Omega(n)}$ subset of $\zo^n$. This implies $\|\Dstar\|_2^2 \le 2^{-\Omega(n)}$, so our reduction applies to most of the interesting parameter regimes.

\subsection{Implementing Query-Based Learning Algorithms Using MQ-SQs}
\label{sec:MQ-SQ-examples}
We note that many existing MQ-based learning algorithms (or components thereof) for the uniform distribution $\calU$ over the hypercube $\calX = \zo^n$ can be implemented using MQ-SQs.

\paragraph{Influence estimation.} The influence of variable $x_i$ with respect to $f: \zo^n \to \zo$ is defined as $\pr{x \sim \calU}{f(x) \ne f(x^{\oplus i})}$, and can be equivalently written as
\begin{align*}
    \Ex{x \sim \calU}{f(x)} + \Ex{x \sim \calU}{f(x^{\oplus i})} - 2\Ex{x \sim \calU}{f(x)f(x^{\oplus i})}
=   2\Ex{x \sim \calU}{f(x)} - 2\Ex{x \sim \calU}{f(x)f(x^{\oplus i})}
\end{align*}
using the identity $\1{b_1 \ne b_2} = b_1 + b_2 - 2b_1b_2$ for $b_1, b_2 \in \zo$. The first term can be evaluated using a Type~I query by setting $\phi(x) \equiv 1$ and $\Dstar = \calU$. The second corresponds to a Type~II query with $\phi(x) \equiv 1$, $\Dstar = \calU$, and $\pi: x \mapsto x^{\oplus i}$. An MQ-SQ oracle with tolerance $\tau$ then allows us to estimate the influence up to an additive error of $4\tau$.

Similarly, we can estimate the influences of any restriction of $f$ by letting $\Dstar$ be uniform over the corresponding subcube of $\zo^n$. In particular, if the restriction fixes $k \le n - \Omega(n)$ variables, the subcube is of size $2^{n-k} \ge 2^{\Omega(n)}$, and the squared $2$-norm of $\Dstar$ is $2^{-\Omega(n)}$.

\paragraph{Learning $k$-juntas.} As a direct application, since every relevant variable of a $k$-junta $f$ has an influence of at least $2^{-(k-1)}$, we can identify all the relevant variables of $f$ using an MQ-SQ oracle with tolerance $O(2^{-k})$. To learn $f$, it remains to evaluate $\Ex{x \sim \Dstar}{f(x)}$ where $\Dstar$ is the uniform distribution over a subcube induced by each assignment of the relevant variables. Assuming $k \le n - \Omega(n)$, this step only involves MQ-SQs of Type~I where $\|\Dstar\|_2^2 = 2^{-\Omega(n)}$.

\paragraph{Proper learning of decision trees.} As another application, we examine a recent query-based algorithm of Blanc, Lange, Qiao, and Tan~\cite[Fig.~2]{BLQT22} that learns decision trees \emph{properly} (i.e., the learned classifier is represented as a decision tree). The algorithm runs in time $n^{O(\log\log n)}$ when learning size-$\poly(n)$ decision trees to a constant error $\eps = \Omega(1)$. The algorithm uses membership queries to estimate the following three quantities for a generic subcube $S \subseteq \zo^n$ of co-dimension $O(\log n)$: (1) The mean of $f$ restricted to $S$: $\Ex{x \sim \calU}{f(x) \mid x \in S}$; (2) The error of  a candidate decision tree $\widehat T$ over $S$: $\pr{x \sim \calU}{f(x) \ne \widehat T(x) \mid x \in S}$; (3) The influence of the $i$-th variable with respect to $f$ restricted to $S$: $\pr{x \sim \calU}{f(x) \ne f(x^{\oplus i}) \mid x \in S}$. The first two quantities can be estimated using Type~I MQ-SQs. As we showed earlier, influences with respect to restrictions of $f$ can be estimated using MQ-SQs of Types I~and~II. Since $|S| = 2^{n - O(\log n)} = 2^{\Omega(n)}$, $\|\Dstar\|_2^2 = 2^{-\Omega(n)}$ also holds for all MQ-SQs made by the algorithm.

\paragraph{The Kushilevitz-Mansour algorithm.} The algorithm of Kushilevitz and Mansour~\cite{KM91} learns an $n$-variable boolean function $f$ with at most $s$ non-zero Fourier coefficients to error $\eps$ using membership queries in $\poly(ns/\eps)$ time. For brevity, we assume that $f$ maps $\zo^n$ to $\{\pm 1\}$ instead of $\zo$; to handled the $\{0, 1\}$-valued function, it suffices to apply the transformation $b \mapsto 1 - 2b$, which leads to at most an $O(1)$ blow-up in the number of queries and the additive error. For $S \subseteq [n]$, we write $\chi_S(x) \coloneqq \prod_{i \in S}(-1)^{x_i}$ as the Fourier basis indexed by $S$, and let $\widehat f(S) \coloneqq \Ex{x \sim \calU}{f(x) \cdot \chi_S(x)}$ be the corresponding Fourier coefficient.
    
Following the exhibition of the KM algorithm in~\cite[Section 3.5]{ODonnell14}, the algorithm only uses membership queries in the following two ways, both of which can be implemented using $\poly(ns/\eps)$ MQ-SQs with tolerance $1/\poly(ns/\eps)$:

\begin{itemize}
    \item \textbf{Estimation of Fourier coefficients.} Given $S \subseteq [n]$, the algorithm needs to estimate $\widehat f(S)$ up to an additive error of $1 / \poly(s / \eps)$. This can be done using a Type~I MQ-SQ with $\phi = \chi_S$ and $\Dstar = \calU$.

    \item \textbf{Estimation of total Fourier weights.} Given $S \subseteq J \subseteq [n]$, the algorithm needs to estimate $\sum_{U \subseteq \overline{J}}\left[\widehat f(S \cup U)\right]^2$ up to an additive error of $1 / \poly(s / \eps)$. By \cite[Proposition 3.40]{ODonnell14},
    \[
        \sum_{U \subseteq \overline{J}}\left[\widehat f(S \cup U)\right]^2
    =   \Ex{z \sim \zo^{\overline{J}}}{\Ex{y, y' \sim \zo^J}{f(y, z)\chi_S(y, z) \cdot f(y', z)\chi_S(y', z)}}.
    \]
    Let $\Delta_J$ denote the distribution of $\delta \in \zo^n$ obtained from: (1) For each $i \in J$, pick $\delta_i$ from $\zo$ independently and uniformly at random; (2) For each $i \in \overline{J}$, set $\delta_i = 0$. Then, the distribution of $((y, z), (y', z))$ when sampling $z \sim \zo^{\overline{J}}$ and $y, y' \sim \zo^{J}$ independently is identical to that of $(x, x \oplus \delta)$ for independent $x \sim \calU$ and $\delta \sim \Delta_J$. It follows that
    \[
        \sum_{U \subseteq \overline{J}}\left[\widehat f(S \cup U)\right]^2
    =   \Ex{\delta \sim \Delta_J}{\Ex{x \sim \calU}{\chi_S(x)\chi_S(x \oplus \delta) \cdot f(x)f(x \oplus \delta)}}.
    \]

    For each fixed $\delta$ in the support of $\Delta_J$, the inner expectation can be estimated using an MQ-SQ. When $\delta$ is the zero vector, the term inside the expectation reduces to $[\chi_S(x)]^2[f(x)]^2 = 1$, so the inner expectation takes value $1$. When $\delta$ is non-zero, the mapping $x \mapsto x \oplus \delta$ is a permutation over $\zo^n$ without fixed points, so it is exactly a Type~III MQ-SQ with $\phi(x) \coloneqq \chi_S(x)\chi_S(x \oplus \delta)$, $\Dstar = \calU$, and $\pi: x \mapsto x \oplus \delta$.

    Therefore, it suffices to sample multiple copies of $\delta$ from $\Delta_J$ and compute the empirical average of the inner expectation. For each inner expectation, we query an MQ-SQ oracle with tolerance $1/\poly(s/\eps)$. By a Chernoff bound, sampling $\poly(ns/\eps)$ copies of $\delta$ is sufficient for the estimation error to be $1/\poly(s/\eps)$ except with probability $1/\poly(ns/\eps)$. Since the KM algorithm runs in $\poly(ns/\eps)$ time, at most $\poly(ns/\eps)$ different values need to be estimated. The union bound then implies that, with high probability, the error is small for all the estimates. We make $\poly(ns/\eps)$ MQ-SQs for each quantity $\sum_{U \subseteq \overline{J}}\left[\widehat f(S \cup U)\right]^2$ to be estimated, so the total number of MQ-SQs is also bounded by $\poly(ns/\eps)$.
\end{itemize}

\subsection{MQ-SQ Testable Learning Implies SQ Refutation}
We will show that, if there is an efficient MQ-SQ algorithm that testably learns class $\calC$, the same class can be refuted by an efficient SQ algorithm. To this end, we first recall the definition of a (usual) SQ-based algorithm in the context of refutation (\Cref{def:refutation}). 
As in \Cref{sec:reduction}, we let $\Dref$ denote the distribution of labeled pairs in the refutation instance and $\calD$ denote the unknown marginal distribution in a TL-Q instance.

\begin{definition}[SQ oracle for refutation]
\label{def:SQ-oracle}
    An SQ oracle for refutation with tolerance $\tau \ge 0$ answers queries of form $\Ex{(x, y) \sim \Dref}{\phi(x, y)}$ within an additive error of $\tau$ given a test function $\phi: \calX \times \zo \to [0, 1]$.
\end{definition}

\paragraph{Recap: Reduce refutation to TL-Q.} We start by recalling the reduction from \Cref{sec:reduction}. Let $p \coloneqq \pr{(x, y) \sim \Dref}{y = 1}$ be the fraction of positive labels in the refutation instance. We consider the TL-Q instance on the same concept class $\calC$, where the target function $f: \calX \to \zo$ is chosen as a \emph{random $p$-biased function}, i.e., each function value $f(x)$ is sampled from $\Bern(p)$ independently. The marginal distribution $\calD \in \Delta(\calX)$ is the distribution of the output $x \in \calX$ produced by the following procedure:
\begin{itemize}
    \item Sample $(x, y) \sim \Dref$.
    \item If $y = f(x) = 1$, output $x$ with probability $1 - p$.
    \item If $y = f(x) = 0$, output $x$ with probability $p$.
    \item If no output is produced, return to the first step.
\end{itemize}
Formally, the probability mass function of $\calD$ is given by
\begin{equation}\label{eq:TLQ-dist}
    \calD(x)
=   \frac{1}{Z}\cdot\calD_x(x)\left[(1 - p) \cdot y(x) \cdot f(x) + p \cdot (1 - y(x)) \cdot (1 - f(x))\right],
\end{equation}
where $\calD_x \in \Delta(\calX)$ is the $\calX$-marginal of $\Dref$,
\[
    y(x) \coloneqq \Ex{(x', y') \sim \Dref}{y' \mid x' = x}
\]
is the conditional expectation of $y \mid x$ over $\Dref$, and
\begin{equation}\label{eq:Z}
    Z \coloneqq \sum_{x \in \calX}\calD_x(x)\left[(1 - p) \cdot y(x) \cdot f(x) + p \cdot (1 - y(x)) \cdot (1 - f(x))\right]
\end{equation}
is a normalization factor.

\paragraph{Queries over a customized $\Dstar$.} We claim that we can answer the first two types of MQ-SQs,
\[
    \Ex{x \sim \Dstar}{\phi(x)f(x)}
\quad\text{and}\quad
    \Ex{x \sim \Dstar}{\phi(x)f(x)f(\pi(x))},
\]
by simply replacing each $f(x)$ with its expectation $p$. In other words, we will show that both
\[
    \Ex{x \sim \Dstar}{\phi(x)f(x)}
\approx p \cdot \Ex{x \sim \Dstar}{\phi(x)}
\quad \text{and} \quad
    \Ex{x \sim \Dstar}{\phi(x)f(x)f(\pi(x))}
\approx p^2 \cdot \Ex{x \sim \Dstar}{\phi(x)}
\]
hold up to a small error with high probability, so that both types of queries can be answered using $\phi$ and $\Dstar$, without actually ``realizing'' the random function $f$.

Towards proving the above, we note that for each fixed $\phi$, over the randomness of $f$,
\[
    \Ex{x \sim \Dstar}{\phi(x)f(x)}
=   \sum_{x \in \calX}\Dstar(x)\phi(x)f(x)
\]
is a sum of independent random variables and thus concentrates around its mean. A Type~II query of form
\[
    \Ex{x \sim \Dstar}{\phi(x)f(x)f(\pi(x))}
=   \sum_{x \in \calX}\Dstar(x)\phi(x)f(x)f(\pi(x))
\]
is slightly more complicated, since the summand is not independent for different values of $x \in \calX$. For example, for two different elements $x, x' \in \calX$, the random variables $f(x)f(\pi(x))$ and $f(x')f(\pi(x'))$ are dependent if $x' = \pi(x)$. A workaround is to consider the graph induced by the permutation $\pi$ (formed by all directed edges $x \to \pi(x)$). Since $\phi$ has no fixed points, the resulting graph can be decomposed into disjoint cycles of length $\ge 2$. We can then partition the $|\calX|$ edges into at most three sets, such that no vertex appear twice in edges within each set. This decomposes the summation $\sum_{x \in \calX}\Dstar(x)\phi(x)f(x)f(\pi(x))$ into three terms, each of which is a sum of independent random variables. Applying Hoeffding's inequality to each term separately gives the same concentration result.

Formally, we have the following lemma.

\begin{lemma}[Types I and II]\label{lemma:type-12}
    The following holds for every $\eps \ge 0$, $\phi: \calX \to [0, 1]$, $\Dstar \in \Delta(\calX)$ and permutation $\pi: \calX \to \calX$ without fixed points:
    \begin{itemize}
        \item With probability at least $1 - 2e^{-\Omega(\eps^2 / \|\Dstar\|_2^2)}$ over the randomness of $f$,
        \[
            \left|\Ex{x \sim \Dstar}{\phi(x)f(x)} - p \cdot \Ex{x \sim \Dstar}{\phi(x)}\right| \le \eps.
        \]
        \item With probability at least $1 - 6e^{-\Omega(\eps^2 / \|\Dstar\|_2^2)}$ over the randomness of $f$,
        \[
            \left|\Ex{x \sim \Dstar}{\phi(x)f(x)f(\pi(x))} - p^2 \cdot \Ex{x \sim \Dstar}{\phi(x)}\right| \le \eps.
        \]
    \end{itemize}
\end{lemma}

\begin{proof}
    Recall that the value of $f$ on each input $x \in \calX$ is sampled from $\Bern(p)$ independently. Thus, over the randomness of $f$,
    \[
        X_1 \coloneqq \Ex{x \sim \Dstar}{\phi(x)f(x)}
    =   \sum_{x \in \calX}\Dstar(x)\phi(x)f(x)
    \]
    is the average of $|\calX|$ independent random variables, and each random variable $\Dstar(x)\phi(x)f(x)$ lies in $[0, \Dstar(x)]$. The expectation of $X_1$ is
    \[
        \Ex{}{X_1}
    =   \sum_{x \in \calX}\Dstar(x)\phi(x)\cdot p
    =   p \cdot \Ex{x \sim \Dstar}{\phi(x)}.
    \]
    The first claim then follows from Hoeffding's inequality.

    Towards proving the second claim, we expand the expectation in a similar way:
    \[
        X_2 \coloneqq \Ex{x \sim \Dstar}{\phi(x)f(x)f(\pi(x))}
    =   \sum_{x \in \calX}\Dstar(x)\phi(x)f(x)f(\pi(x)).
    \]
    The expectation of $X_2$ over the randomness of $f$ is given by
    \[
        \mu_2
    \coloneqq \Ex{}{X_2}
    =   \sum_{x \in \calX}\Dstar(x)\phi(x) \cdot \Ex{}{f(x)f(\pi(x))}
    =   \sum_{x \in \calX}\Dstar(x)\phi(x) \cdot p^2
    =   p^2 \cdot \Ex{x \sim \Dstar}{\phi(x)},
    \]
    where the second step applies the assumption that $\pi$ has no fixed points. Note that $\pi$ partitions $\calX$ into disjoint cycles of lengths $\ge 2$, where each cycle is of form $x^{(1)} \to x^{(2)} \to \cdots \to x^{(\ell)} \to x^{(1)}$ such that $\pi(x^{(1)}) = x^{(2)}$, $\pi(x^{(2)}) = x^{(3)}$, $\ldots$, $\pi(x^{(\ell)}) = x^{(1)}$. For each such cycle, we can color its edges with three colors, so that no two edges of the same color share a vertex. Combining the coloring for different cycles yields a partition $\calX = \calX_1 \cup \calX_2 \cup \calX_3$ such that, for every $i \in \{1, 2, 3\}$ and any two different elements $x, x' \in \calX_i$, the four elements $x, \pi(x), x', \pi(x')$ are different.

    Then, we can write $X_2 = X_{2,1} + X_{2,2} + X_{2,3}$ where $X_{2, i} \coloneqq \sum_{x \in \calX_i}\Dstar(x)\phi(x)f(x)f(\pi(x))$ for each $i \in \{1, 2, 3\}$. Letting $\mu_{2,i} \coloneqq \Ex{}{X_{2, i}}$, we have
    \[
        X_2 - \mu_2
    =   (X_{2,1} - \mu_{2,1}) + (X_{2,2} - \mu_{2,2}) + (X_{2,3} - \mu_{2,3}).
    \]
    For $|X_2 - \mu_2| > \eps$ to hold, $|X_{2, i} - \mu_{2, i}| > \frac{\eps}{3}$ must hold for some $i \in \{1, 2, 3\}$. Since each $X_{2,i}$ is the sum of $|\calX_i|$ independent random variables, Hoeffding's inequality gives
    \[
        \pr{}{|X_{2, i} - \mu_{2, i}| > \frac{\eps}{3}}
    \le 2\exp\left(-\frac{2 \cdot (\eps/3)^2}{\sum_{x \in \calX_i}[\Dstar(x)]^2}\right)
    \le 2\exp\left(-\frac{2\eps^2/9}{\|\Dstar\|_2^2}\right)
    \le 2e^{-\Omega(\eps^2 / \|\Dstar\|_2^2)}.
    \]
    The second part of the lemma then follows from the union bound.
\end{proof}

%
%

\paragraph{Queries over the unknown $\calD$.} Now, we handle the other three types of MQ-SQs that are over the unknown marginal distribution $\calD$ in the TL-Q instance:
\[
    \Ex{x \sim \calD}{\phi(x)},
\quad
    \Ex{x \sim \calD}{\phi(x)f(x)},
\quad\text{and}\quad
    \Ex{x \sim \calD}{\phi(x)f(x)f(\pi(x))}.
\]

\begin{lemma}[Types III, IV and V]\label{lemma:type-345}
    The following holds for every $\eps \ge 0$, $\phi: \calX \to [0, 1]$ and permutation $\pi: \calX \to \calX$ without fixed points:
    \begin{itemize}
        \item With probability at least $1 - 4e^{-\Omega(\eps^2\cdot p^2(1-p)^2 / \|\calD_x\|_2^2)}$ over the randomness of $f$,
        \[
            \left|\Ex{x \sim \calD}{\phi(x)} - \Ex{(x, y) \sim \Dref}{\phi(x)}\right| \le \eps.
        \]
        \item With probability at least $1 - 4e^{-\Omega(\eps^2\cdot p^2(1-p)^2 / \|\calD_x\|_2^2)}$ over the randomness of $f$,
        \[
            \left|\Ex{x \sim \calD}{\phi(x)f(x)} - \Ex{(x, y) \sim \Dref}{\phi(x) \cdot y}\right| \le \eps.
        \]
        \item With probability at least $1 - 8e^{-\Omega(\eps^2\cdot p^2(1-p)^2 / \|\calD_x\|_2^2)}$ over the randomness of $f$,
        \[
            \left|\Ex{x \sim \calD}{\phi(x)f(x)f(\pi(x))} - p \cdot \Ex{(x, y) \sim \Dref}{\phi(x) \cdot y}\right| \le \eps.
        \]
    \end{itemize}
\end{lemma}

\begin{proof}
    Recall that the probability mass function of $\calD$ is given by
    \[
        \calD(x) = \frac{1}{Z}\cdot\calD_x(x)\left[(1 - p) \cdot y(x) \cdot f(x) + p \cdot(1 - y(x))\cdot (1 - f(x))\right],
    \]
    where $y(x) = \Ex{(x', y') \sim \Dref}{y' \mid x' = x}$ and $Z$ is a normalization factor that concentrates around $p(1-p)$ (\Cref{lemma:Z-concentration}).

    \paragraph{The first claim: $\Ex{x \sim \calD}{\phi(x)} \approx \Ex{(x, y) \sim \Dref}{\phi(x)}$.} Expanding this expectation gives
    \begin{align*}
        \Ex{x \sim \calD}{\phi(x)}
    &=  \sum_{x \in \calX}\calD(x)\phi(x)\\
    &=  \frac{1}{Z}\sum_{x \in \calX}\calD_x(x)\phi(x)\cdot\left[(1 - p) \cdot y(x) \cdot f(x) + p \cdot(1 - y(x))\cdot (1 - f(x))\right]\\
    &=  \frac{1}{Z}\left[p\cdot\sum_{x \in \calX}\calD_x(x)\phi(x)(1 - y(x)) + \sum_{x \in \calX}\calD_x(x)\phi(x)(y(x) - p)f(x)\right]\\
    &=  \frac{X_1}{\alpha \cdot p(1-p)},
    \end{align*}
    where we use shorthands $\alpha \coloneqq \frac{Z}{p(1-p)}$ and
    \[
        X_1 \coloneqq p \cdot \sum_{x \in \calX}\calD_x(x)\phi(x)(1-y(x)) + \sum_{x \in \calX}\calD_x(x)\phi(x)(y(x) - p)f(x)
    \]
    in the last step.
    
    Over the randomness of $f$, $X_1$ is a deterministic value plus the sum of $|\calX|$ independent random variables, where the summand that corresponds to each $x \in \calX$ always lies in $[-\calD_x(x), \calD_x(x)]$. The expectation of $X_1$ is given by
    \[
        \mu_1 \coloneqq \Ex{}{X_1}
    =   p \cdot \sum_{x \in \calX}\calD_x(x)\phi(x)(1-y(x)) + \sum_{x \in \calX}\calD_x(x)\phi(x)(y(x) - p) \cdot p
    =   p(1-p) \cdot \Ex{x \sim \calD_x}{\phi(x)}.
    \]
    By Hoeffding's probability, for every $\eps_1 \ge 0$, it holds with probability at least $1 - 2e^{-\Omega(-\eps_1^2 / \|\calD_x\|_2^2)}$ that
    \begin{equation}\label{eq:type-3-cond-1}
        |X_1 - \mu_1| \le \eps_1.
    \end{equation}
    By \Cref{lemma:Z-concentration}, for every $\eps_2 \in [0, 1/2]$, it holds with probability at least $1 - 2e^{-\Omega(\eps_2^2\cdot p^2(1-p)^2 / \|\calD_x\|_2^2)}$ that
    \begin{equation}\label{eq:type-3-cond-2}
        \alpha = \frac{Z}{p(1-p)} \in [1 - \eps_2, 1 + \eps_2].
    \end{equation}

    Recall that our goal is to derive a high-probability upper bound on
    \begin{align*}
        \left|\Ex{x \sim \calD}{\phi(x)} - \Ex{(x, y) \sim \Dref}{\phi(x)}\right|
    =   &~\left|\frac{X_1}{\alpha \cdot p(1-p)} - \Ex{x \sim \calD_x}{\phi(x)}\right|\\
    =   &~\left|\frac{X_1}{\alpha \cdot p(1-p)} - \frac{\mu_1}{p(1-p)}\right|\\
    \le &~\left|\frac{X_1}{\alpha \cdot p(1-p)} - \frac{\mu_1}{\alpha \cdot p(1-p)}\right| + \left|\frac{\mu_1}{\alpha \cdot p(1-p)} - \frac{\mu_1}{p(1-p)}\right|.
    \end{align*}
    When both \eqref{eq:type-3-cond-1}~and~\eqref{eq:type-3-cond-2} hold, the first term above is at most
    \[
        \frac{1}{\alpha} \cdot \frac{1}{p(1-p)}\cdot|X_1 - \mu_1|
    \le 2 \cdot \frac{1}{p(1-p)} \cdot \eps_1
    =   \frac{2\eps_1}{p(1-p)},
    \]
    while the second term reduces to
    \[
        \Ex{x \sim \calD_x}{\phi(x)} \cdot \left|\frac{1}{\alpha} - 1\right|
    \le \left|\frac{1}{\alpha} - 1\right|
    =   \frac{|\alpha - 1|}{\alpha}
    \le 2|\alpha - 1|
    \le 2\eps_2.
    \]
    Here, we use the assumption that $|\alpha - 1| \le \eps_2 \le 1/2$. Setting $\eps_1 = \eps\cdot p(1-p) / 4$ and $\eps_2 = \eps / 4$, the first claim follows from the union bound.

    \paragraph{The second claim: $\Ex{x \sim \calD}{\phi(x)f(x)} \approx \Ex{(x, y) \sim \Dref}{\phi(x) \cdot y}$.} Similarly, we expand this expectation as follows:
    \begin{align*}
        \Ex{x \sim \calD}{\phi(x)f(x)}
    &=  \sum_{x \in \calX}\calD(x)\phi(x)f(x)\\
    &=  \frac{1}{Z}\sum_{x \in \calX}\calD_x(x)\phi(x)f(x)\cdot\left[(1 - p) \cdot y(x) \cdot f(x) + p \cdot(1 - y(x))\cdot (1 - f(x))\right]\\
    &=  \frac{1 - p}{Z}\sum_{x \in \calX}\calD_x(x)\phi(x)y(x) \cdot f(x)\\
    &=  \frac{X_2}{\alpha \cdot p},
    \end{align*}
    where the third step applies $[f(x)]^2 = f(x)$ and $f(x)(1 - f(x)) = 0$, and the last step introduces shorthands $\alpha \coloneqq \frac{Z}{p(1-p)}$ and
    \[
        X_2 \coloneqq \sum_{x \in \calX}\calD_x(x)\phi(x)y(x) \cdot f(x).
    \]

    Over the randomness of $f$, $X_2$ is a sum of $|\calX|$ independent random variables, and has an expectation of
    \[
        \mu_2 \coloneqq \Ex{}{X_2}
    =   \sum_{x \in \calX}\calD_x(x)\phi(x)y(x) \cdot \Ex{}{f(x)}
    =   p \cdot \sum_{x \in \calX}\calD_x(x)\phi(x)y(x)
    =   p \cdot \Ex{(x, y) \sim \Dref}{\phi(x) \cdot y}.
    \]
    By Hoeffding's inequality, for every $\eps_1 \ge 0$, it holds with probability at least $1 - 2e^{-\Omega(\eps_1^2 / \|\calD_x\|_2^2)}$ that $|X_2 - \mu_2| \le \eps_1$. By \Cref{lemma:Z-concentration}, for every $\eps_2 \in [0, 1/2]$, it holds with probability at least $1 - 2e^{-\Omega(\eps_2^2\cdot p(1-p) / \|\calD_x\|_2^2)}$ that $\alpha \in [1 - \eps_2, 1 + \eps_2]$.

    Recall that our goal is to upper bound the difference between $\Ex{x \sim \calD}{\phi(x)f(x)} = X_2 / (\alpha \cdot p)$ and $\Ex{(x, y) \sim \Dref}{\phi(x) \cdot y} = \mu_2 / p$. By the triangle inequality, we have
    \[
        \left|\Ex{x \sim \calD}{\phi(x)f(x)} - \Ex{(x, y) \sim \Dref}{\phi(x) \cdot y}\right|
    \le \left|\frac{X_2}{\alpha \cdot p} - \frac{\mu_2}{\alpha \cdot p}\right| + \left|\frac{\mu_2}{\alpha \cdot p} - \frac{\mu_2}{p}\right|.
    \]
    Assuming $|X_2 - \mu_2| \le \eps_1$ and $\alpha \ge 1 - \eps_2 \ge 1/2$, the first term above is at most
    \[
        \frac{1}{\alpha} \cdot \frac{1}{p} \cdot |X_2 - \mu_2|
    \le \frac{2\eps_1}{p},
    \]
    while the second term reduces to
    \[
        \frac{\mu_2}{p} \cdot \left|\frac{1}{\alpha} - 1\right|
    =   \Ex{(x, y) \sim \Dref}{\phi(x) \cdot y} \cdot \left|\frac{1}{\alpha} - 1\right|
    \le 2|\alpha - 1|
    \le 2\eps_2.
    \]
    The second claim then follows from the union bound if we set $\eps_1 = \eps p / 4$ and $\eps_2 = \eps / 4$.
    
    \paragraph{The third claim: $\Ex{x \sim \calD}{\phi(x)f(x)f(\pi(x))} \approx p \cdot \Ex{(x, y) \sim \Dref}{\phi(x) \cdot y}$.} Similar to the previous two claims, we write
    \begin{align*}
        \Ex{x \sim \calD}{\phi(x)f(x)f(\pi(x))}
    &=  \sum_{x \in \calX}\calD(x)\phi(x)f(x)f(\pi(x))\\
    &=  \frac{1}{Z}\sum_{x \in \calX}\calD_x(x)\phi(x)f(x)f(\pi(x))\cdot\left[(1 - p) \cdot y(x) \cdot f(x) + p \cdot(1 - y(x))\cdot (1 - f(x))\right]\\
    &=  \frac{1 - p}{Z}\sum_{x \in \calX}\calD_x(x)\phi(x)y(x) \cdot f(x)f(\pi(x)) \tag{$f(x) \in \{0, 1\}$}\\
    &=  \frac{X_3}{\alpha \cdot p},
    \end{align*}
    where $\alpha \coloneqq \frac{Z}{p(1-p)}$ and
    \[
        X_3 \coloneqq \sum_{x \in \calX}\calD_x(x)\phi(x)y(x) \cdot f(x)f(\pi(x)).
    \]
    The expectation of $X_3$ is given by
    \[
        \mu_3 \coloneqq \Ex{}{X_3}
    =   \sum_{x \in \calX}\calD_x(x)\phi(x)y(x) \cdot \Ex{}{f(x)f(\pi(x))}
    =   p^2\sum_{x \in \calX}\calD_x(x)\phi(x)y(x)
    =   p^2\cdot\Ex{(x, y) \sim \Dref}{\phi(x) \cdot y},
    \]
    where the last step holds since $\pi$ has no fixed points.

    Towards analyzing the concentration of $X_3$, an obstacle is that the $|\calX|$ random variables $\{f(x) f(\pi(x)): x \in \calX\}$ are not independent. As in the proof of \Cref{lemma:type-12}, our workaround is to partition them into three parts, each of which consists of independent random variables. By taking a three-coloring of the cycle decomposition of $\pi$, we obtain a partition $\calX = \calX_1 \cup \calX_2 \cup \calX_3$ with the property that, for every $i \in \{1, 2, 3\}$ and two different elements $x, x' \in \calX_i$,
    \[
        |\{x, \pi(x), x', \pi(x')\}| = 4.
    \]
    This in turn implies that, for each $i$, the random variables $\{f(x)f(\pi(x)): x \in \calX_i\}$ are independent.

    Then, for each $i \in \{1, 2, 3\}$, we define
    \[
        X_{3, i} \coloneqq \sum_{x \in \calX_i}\calD_x(x)\phi(x)y(x) \cdot f(x)f(\pi(x))
    \]
    and let $\mu_{3, i} \coloneqq \Ex{}{X_{3, i}}$ be its mean. Clearly, it holds that $X_3 = X_{3,1} + X_{3,2} + X_{3,3}$ and $\mu_3 \coloneqq \Ex{}{X_3} = \mu_{3,1} + \mu_{3,2} + \mu_{3,3}$. Applying Hoeffding's inequality to each $X_{3,i}$ and then applying the union bound gives, for every $\eps_1 \ge 0$,
    \begin{align*}
        \pr{}{|X_3 - \mu_3| \le \eps_1}
    &\ge1 - \sum_{i=1}^{3}\pr{}{|X_{3,i} - \mu_{3,i}| \ge \frac{\eps_1}{3}}\\
    &\ge1 - \sum_{i=1}^{3}2\exp\left(-\Omega\left(\frac{\eps_1^2}{\sum_{x \in \calX_i}[\calD_x(x)]^2}\right)\right)
    \ge 1 - 6e^{-\Omega(\eps_1^2 / \|\calD_x\|_2^2)}.
    \end{align*}
    Furthermore, by \Cref{lemma:Z-concentration}, it holds for every $\eps_2 \in [0, 1/2]$ that $\pr{}{\alpha \in [1 - \eps_2, 1 + \eps_2]} \ge 1 - 2e^{-\Omega(\eps_2^2\cdot p(1-p) / \|\calD_x\|_2^2)}$.

    Again, we bound the difference between $\Ex{x \sim \calD}{\phi(x)f(x)f(\pi(x))} = \frac{X_3}{\alpha \cdot p}$ and $p \cdot \Ex{(x, y) \sim \Dref}{\phi(x) \cdot y} = \mu_3 / p$ using the triangle inequality:
    \[
        \left|\Ex{x \sim \calD}{\phi(x)f(x)f(\pi(x))} - p \cdot \Ex{(x, y) \sim \Dref}{\phi(x) \cdot y}\right|
    \le \left|\frac{X_3}{\alpha \cdot p} - \frac{\mu_3}{\alpha \cdot p}\right| + \left|\frac{\mu_3}{\alpha \cdot p} - \frac{\mu_3}{p}\right|.
    \]
    Assuming $|X_3 - \mu_3| \le \eps_1$ and $\alpha \in [1 - \eps_2, 1 + \eps_2]$, the first term above is at most
    \[
        \frac{1}{\alpha} \cdot \frac{1}{p} \cdot |X_3 - \mu_3|
    \le \frac{2\eps_1}{p},
    \]
    while the second term is at most
    \[
        \frac{\mu_3}{p} \cdot \left|\frac{1}{\alpha} - 1\right|
    =   p \cdot \Ex{(x, y) \sim \Dref}{\phi(x) \cdot y} \cdot \left|\frac{1}{\alpha} - 1\right|
    \le 2|\alpha - 1|
    \le 2\eps_2.
    \]
    Setting $\eps_1 = p\eps / 4$ and $\eps_2 = \eps / 4$ proves the third claim.
\end{proof}

Now, we put everything together and prove our main result on MQ-SQ algorithms for testable learning.

\begin{proposition}\label{prop:MQ-SQ-TLQ-implies-SQ-refutation}
    Let $\calC$ be a concept class of boolean functions over instance space $\calX$. Suppose that there is a $(c, \eps, \delta)$-PAC MQ-SQ algorithm that testably learns $\calC$ over distribution family $\calF \subseteq \Delta(\calX)$ using at most $q$ queries to an MQ-SQ oracle with tolerance $\tau > 0$. Then, there is an algorithm that solves biased-$(\alpha, \eta)$-refutation on $\calC$ over the same distribution family $\calF$ by making at most $q' = q + O(1)$ queries to an SQ oracle with tolerance $\tau' = \tau / 4$ and has a failure probability of at most $\delta' = \delta + O(q) \cdot e^{-\Omega(\tau^2 B)}$, assuming the following:
    \begin{enumerate}
        \item$\alpha > c\eta + \eps + (c + 4)\tau + 6\tau'$;
        \item $B \le p^2(1-p)^2 / \|\calD_x\|_2^2$ for every $\calD_x \in \calF$, where $p = \Ex{(x, y) \sim \Dref}{y}$ is the average label in the refutation instance;
        \item $B \le 1 / \|\Dstar\|_2^2$ holds for all MQ-SQs of Types I~and~II that $\calA$ makes.
    \end{enumerate}
\end{proposition}

\begin{proof}
    Let $\calA$ denote the hypothetical MQ-SQ algorithm that testably learns $\calC$. We construct a new algorithm, denoted by $\calA'$, that refutes $\calC$ using an SQ oracle by simulating the execution of $\calA$.

    Recall that we defined $p = \Ex{(x, y) \sim \Dref}{y}$. As the first step of the algorithm, $\calA'$ queries the SQ oracle (\Cref{def:SQ-oracle}) with $\phi(x, y) = y$ to obtain an estimate $\widehat p \in [p - \tau', p + \tau']$ for $p$.
   
    \paragraph{Handling queries.} Whenever the simulated copy of $\calA$ makes an MQ-SQ, $\calA'$ answers the query as follows:
    \begin{itemize}
        \item When $\calA$ makes a Type~I query on $\Ex{x \sim \Dstar}{\phi(x)f(x)}$, $\calA'$ returns $\widehat p \cdot \Ex{x \sim \Dstar}{\phi(x)}$.
        \item When $\calA$ makes a Type~II query on $\Ex{x \sim \Dstar}{\phi(x)f(x)f(\pi(x))}$, $\calA'$ returns ${\widehat p}^2 \cdot \Ex{x \sim \Dstar}{\phi(x)}$.
        \item When $\calA$ makes a Type~III query on $\Ex{x \sim \calD}{\phi(x)}$, $\calA'$ queries the SQ oracle on the value of $\mu = \Ex{(x, y) \sim \Dref}{\phi(x)}$ and then forwards the answer $\widehat\mu \in [\mu - \tau', \mu + \tau']$ to $\calA$.
        \item When $\calA$ makes a Type~IV query on $\Ex{x \sim \calD}{\phi(x)f(x)}$, $\calA'$ queries the SQ oracle on the value of $\mu = \Ex{(x, y) \sim \Dref}{\phi(x) \cdot y}$ and then forwards the answer $\widehat\mu \in [\mu - \tau', \mu + \tau']$ to $\calA$.
        \item When $\calA$ makes a Type~V query on $\Ex{x \sim \calD}{\phi(x)f(x)f(\pi(x))}$, $\calA'$ queries the SQ oracle on the value of $\mu = \Ex{(x, y) \sim \Dref}{\phi(x) \cdot y}$ and then forwards the answer $\widehat\mu \in [\mu - \tau', \mu + \tau']$ multiplied by $\widehat p$ to $\calA$.
    \end{itemize}
    In the first two cases, $\calA'$ can exactly compute the expectations since it has full knowledge of $\phi: \calX \to [0, 1]$ and $\Dstar \in \Delta(\calX)$.

    \paragraph{Decision rule.} When $\calA$ terminates, $\calA'$ decides on the refutation instance as follows:
    \begin{itemize}
        \item If $\calA$ rejects the TL-Q instance, $\calA'$ returns $\structure$, indicating that some $f^* \in \calC$ has an error $\le \eta$ on $\Dref$.
        \item If $\calA$ accepts and returns a function $\widehat f: \calX \to \zo$, $\calA'$ makes the following three additional MQ-SQs on behalf of $\calA$ and answer them as described above:
        \[
            \Ex{x \sim \calD}{\widehat f(x)},
        \quad
            \Ex{x \sim \calD}{f(x)},
        \quad \text{and} \quad
            \Ex{x \sim \calD}{\widehat f(x) f(x)}.
        \]
        Let $\widehat\mu_1, \widehat\mu_2, \widehat\mu_3$ denote the answers to the three queries, and let
        \[
            \widehat \mu = \widehat\mu_1 + \widehat\mu_2 - 2\widehat\mu_3
        \]
        be a weighted sum of them. Note that $\widehat \mu$ is intended to be an estimate of
        \[
            \Ex{x \sim \calD}{\widehat f(x) + f(x) - 2\widehat f(x) f(x)}
        =   \pr{x \sim \calD}{\widehat f(x) \ne f(x)}.
        \]
        Finally, $\calA'$ outputs $\noise$ (indicating that the labels are random) if $\widehat\mu \ge \min\{\widehat p, 1 - \widehat p\} - 5\tau'$, and outputs $\structure$ otherwise.
    \end{itemize}

    \paragraph{Overview of analysis.} We first upper bound the number of SQs that $\calA'$ makes. By construction, $\calA'$ queries the SQ oracle at most once for every MQ-SQ made by $\calA$. In addition, $\calA'$ makes one query at the beginning and at most three queries at the end. Thus, $\calA'$ makes at most $q' = q + O(1)$ queries in total.

    To analyze the correctness of $\calA'$, let $f: \calX \to \zo$ be a random $p$-biased function obtained by independently drawing the function value $f(x)$ from $\Bern(p)$ for each $x \in \calX$. Note that $f$ is only for the analysis; it is never used in algorithm $\calA'$. Also, let $\calD$ denote the distribution over $\calX$ induced by $\Dref$ and $f$ (see \Cref{eq:TLQ-dist}). We will first argue that, with high probability, the simulated copy of $\calA$ effectively runs on an instance of testable learning with target function $f$ and marginal distribution $\calD$. We will then show that the decision made by $\calA'$ is correct due to the intended behavior of $\calA$ on such an instance.

    \paragraph{A good event.} Let $\Egood$ be the ``good event'' that the three conditions below hold simultaneously:
    \begin{itemize}
        \item The simulated execution of $\calA$ coincides with its execution on the testable learning instance $(f, \calD)$ using an MQ-SQ oracle with tolerance $\tau$. In other words, every MQ-SQ made by $\calA$ is answered up to an additive error of $\tau$.
        \item If the first condition holds, the output of $\calA$ is valid with respect to the testable learning instance.
        \item If there exists $f^* \in \calC$ that satisfies $\pr{(x, y) \sim \Dref}{f^*(x) \ne y} \le \eta$ (i.e., we are in the $\structure$ case), it holds that $\pr{x \sim \calD}{f^*(x) \ne f(x)} \le \eta + \tau$.
    \end{itemize}
    By \Cref{lemma:type-12,lemma:type-345}, for each MQ-SQ made by $\calA$, the first condition gets violated with probability at most $8e^{-\Omega(\tau^2B)}$, where $B$ is the minimum between the value of $1/\|\Dstar\|_2^2$ (among all queries of Types I~and~II) and $p^2(1-p)^2 / \|\calD_x\|_2^2$. Applying the union bound to the $\le q + 4$ queries shows that the first condition gets violated with probability at most $O(q) \cdot e^{-\Omega(\tau^2B)}$. Since $\calA$ is assumed to be $(c, \eps, \delta)$-PAC, the second condition gets violated with probability at most $\delta$. Applying \Cref{clm:error-blowup-better} with $\delta = \tau$ shows that the third condition gets violated with probability at most $e^{-\Omega(\tau^2p^2(1-p)^2 / \|\calD_x\|_2^2)} \le e^{-\Omega(\tau^2B)}$. Applying the union bound again gives
    \[
        \pr{}{\Egood} \ge 1 - \delta - O(q) \cdot e^{-\Omega(\tau^2 B)}.
    \]
    In the rest of the proof, we show that event $\Egood$ implies that $\calA'$ decides correctly.

    \paragraph{Proof of completeness.} Suppose that some $f^* \in \calC$ satisfies $\pr{(x, y) \sim \Dref}{f^*(x) \ne y} \le \eta$, where $\eta$ is the parameter of the refutation instance (\Cref{def:refutation}). The third condition of the good event $\Egood$ implies that the same $f^*$ has with an error $\le 2\eta$ over $\calD$. Then, the testable learner $\calA$ may output either $\bot$ or a function $\widehat f: \calX \to \zo$ that satisfies
    \[
        \pr{x \sim \calD}{\widehat f(x) \ne f(x)} \le c \cdot \pr{x \sim \calD}{f^*(x) \ne f(x)} + \eps \le c(\eta + \tau) + \eps.
    \]

    In the former case, $\calA'$ would correctly output $\structure$. For the latter case, applying the identity $\1{b_1 \ne b_2} = b_1 + b_2 - 2b_1b_2$ for $b_1, b_2 \in \zo$ gives
    \[
        \pr{x \sim \calD}{\widehat f(x) \ne f(x)}
    =   \Ex{x \sim \calD}{\widehat f(x)} + \Ex{x \sim \calD}{f(x)} - 2\Ex{x \sim \calD}{\widehat f(x) f(x)}.
    \]
    Assuming event $\Egood$, $\calA'$ obtains an estimate of each of the three expectations on the right-hand side above up to an additive error of $\tau$. It then follows that the value of $\widehat\mu$ computed at the end satisfies
    \[
        \widehat\mu \le \pr{x \sim \calD}{\widehat f(x) \ne f(x)} + 4\tau \le c\eta + \eps + (c + 4)\tau.
    \]
    Since $\min\{p, 1 - p\} \ge \alpha > c\eta + \eps + (c + 4)\tau + 6\tau'$, we have
    \[
        \widehat\mu
    <   \min\{p, 1-p\} - 6\tau'
    \le \min\{\widehat p, 1 - \widehat p\} - 5\tau'
    \]
    in this case, and $\calA'$ would correctly output $\structure$.

    \paragraph{Proof of soundness.} Suppose that the distribution $\Dref$ in the refutation instance is the product distribution of some $\calD_x \in \calF$ and $\Bern(p)$. Then, the resulting marginal distribution $\calD$ in the testable learning instance is exactly $\calD_x \in \calF$. Thus, assuming that $\calA$ is correct, $\calA$ would accept and output a function $\widehat f: \calX \to \zo$. Then, at the end of $\calA'$, we compute
    \[
        \widehat\mu \approx \Ex{x \sim \calD}{\widehat f(x)} + \Ex{x \sim \calD}{f(x)} - 2\Ex{x \sim \calD}{\widehat f(x) f(x)}.
    \]
    By the way in which $\calA'$ handles the MQ-SQs, the three terms
    \[
        \Ex{x \sim \calD}{\widehat f(x)},
    \Ex{x \sim \calD}{f(x)}, \text{and}
    \Ex{x \sim \calD}{\widehat f(x)f(x)}
    \]
    are approximated with
    \[
        \Ex{(x, y) \sim \Dref}{\widehat f(x)},
    \Ex{(x, y) \sim \Dref}{y}, \text{and}
    \Ex{(x, y) \sim \Dref}{\widehat f(x) \cdot y},
    \]
    respectively. All the three values are obtained from querying the SQ oracle. Since the SQ oracle has a tolerance of $\tau'$, the value of $\widehat\mu$ is within an additive error of $4\tau'$ to
    \[
        \Ex{(x, y) \sim \Dref}{\widehat f(x) + y - 2\widehat f(x) \cdot y}
    =   \pr{(x, y) \sim \Dref}{\widehat f(x) \ne y},
    \]
    which is exactly the error of $\widehat f$ on distribution $\Dref$. Since $\Dref$ is the product of $\calD_x$ and $\Bern(p)$, regardless of the choice of $\widehat f$, $\pr{(x, y) \sim \Dref}{\widehat f(x) \ne y}$ is at least $\min\{p, 1 - p\}$. Together with the fact that $|p - \widehat p| \le \tau'$, this further implies
    \[
        \widehat\mu \ge \min\{p, 1-p\} - 4\tau' \ge \min\{\widehat p, 1 - \widehat p\} - 5\tau'.
    \]
    Thus, $\calA'$ would correctly output $\noise$.
\end{proof}

\subsection{SQ Refutation Implies SQ Weak Learning}
We show that an SQ algorithm for refutation implies an SQ algorithm for weakly learning the same concept class. Later, using the equivalence between SQ dimension and weak learning~\cite{BFJKMR94}, we can lift SQ-dimension lower bounds to lower bounds against SQ-based refutation. By \Cref{prop:MQ-SQ-TLQ-implies-SQ-refutation}, this further leads to lower bounds against MQ-SQ algorithms for testable learning with queries.

We first recall the definition of SQ-based weak learning.
\begin{definition}[SQ Weak learning]
\label{def:weak-learning}
    Let $\calC \subseteq \{f: \calX \to \zo\}$ be a concept class over a finite instance space $\calX$. Let $\calD$ be a given distribution over $\calX$ and $f^* \in \calC$ be an unknown target function. An SQ oracle for $(f^*, \calD)$ with tolerance $\tau \ge 0$ answers queries of form $\Ex{x \sim \calD}{\phi(x) \cdot f^*(x)}$ up to an additive error of $\tau$, where $\phi: \calX \to [0, 1]$ is any given test function. An SQ algorithm $\eps$-weakly learns $\calC$ if it, by making queries to an SQ oracle, with probability $\ge 2/3$ outputs a classifier $\widehat f: \calX \to \zo$ that satisfies
    \[
        \pr{x \sim \calD}{\widehat f(x) \ne f^*(x)} \le \frac{1}{2} - \eps.
    \]
\end{definition}

Intuitively, to solve refutation (\Cref{def:refutation}) using an SQ algorithm, in the ``structure'' case that some $f^* \in \calC$ has a low error, the algorithm must query the SQ oracle using a test function that has a non-trivial correlation with $f^*$. Such a test function would then allow us to learn the unknown function up to an error that is better than random guessing.

\begin{proposition}
\label{prop:SQ-refutation-implies-SQ-weak-learning}
    Suppose that an SQ algorithm solves biased-$(\alpha, \eta)$-refutation for concept class $\calC$ on distribution $\calD$ by making at most $q$ queries with tolerance $\tau \ge 0$. Then, assuming that $\alpha \le 1/2 - \Omega(\tau)$, there is an SQ algorithm that $\Omega(\tau)$-weakly learns $\calC$ on $\calD$ by making at most $q' = O(q + 1/\tau)$ queries to an SQ oracle with tolerance $\tau' = \Omega(\tau)$.
\end{proposition}

To prove the proposition, we will use the following simple fact: If a $[0, 1]$-valued function has a non-trivial correlation with a sufficiently balanced boolean function, it can be rounded into a random binary classifier with a non-trivial accuracy in expectation. (If the boolean function is far from balanced, it can be easily learned by a constant function.)

\begin{lemma}
\label{lemma:rounding}
    The following holds for every $\delta \ge 0$, distribution $\calD$ over $\calX$, and binary function $f^*: \calX \to \zo$ with mean $p \coloneqq \Ex{x \sim \calD}{f^*(x)} \in [1/2 - \gamma, 1/2 + \gamma]$: Suppose that function $\phi: \calX \to [0, 1]$ satisfies $\Ex{x \sim \calD}{\phi(x) f^*(x)} - p \cdot \Ex{x \sim \calD}{\phi(x)} \ge \delta$. Then, for the random function $\widetilde\phi: \calX \to \zo$ obtained from sampling each $\widetilde\phi(x)$ from $\Bern(\phi(x))$ independently, we have
    \[
        \Ex{\widetilde\phi}{\pr{x \sim \calD}{\widetilde\phi(x) \ne f^*(x)}} \le \frac{1}{2} - (2\delta - 3\gamma).
    \]
    Similarly, if $\Ex{x \sim \calD}{\phi(x) f^*(x)} - p \cdot \Ex{x \sim \calD}{\phi(x)} \le -\delta$, we have
    \[
        \Ex{\widetilde\phi}{\pr{x \sim \calD}{1 - \widetilde\phi(x) \ne f^*(x)}} \le \frac{1}{2} - (2\delta - 3\gamma).
    \]
\end{lemma}

\begin{proof}
    It suffices to prove the first part; the second part follows by symmetry. By the identity $\1{b_1 \ne b_2} = b_1 + b_2 - 2b_1b_2$ for $b_1, b_2 \in \zo$, it holds for every possible realization of $\widetilde\phi: \calX \to \zo$ that
    \[
        \pr{x \sim \calD}{\widetilde\phi(x) \ne f^*(x)}
    =   \Ex{x \sim \calD}{\widetilde\phi(x) + f^*(x) - 2\widetilde\phi(x)f^*(x)}.
    \]
    Taking an expectation over the randomness of $\widetilde\phi$ shows that 
    \begin{align*}
        \Ex{\widetilde\phi}{\pr{x \sim \calD}{\widetilde\phi(x) \ne f^*(x)}}
    &=  \Ex{\widetilde\phi}{\Ex{x \sim \calD}{\widetilde\phi(x) + f^*(x) - 2\widetilde\phi(x)f^*(x)}}\\
    &=  \Ex{x \sim \calD}{\phi(x) + f^*(x) - 2\phi(x)f^*(x)}\\
    &=  \Ex{x \sim \calD}{\phi(x)} + p - 2\Ex{x \sim \calD}{\phi(x)f^*(x)}\\
    &\le p + \Ex{x \sim \calD}{\phi(x)} - 2\left(p\cdot\Ex{x \sim \calD}{\phi(x)} + \delta\right)\\
    &=  p + (1 - 2p)\cdot\Ex{x \sim \calD}{\phi(x)} - 2\delta,
    \end{align*}
    where the fourth step applies the assumption that $\Ex{x \sim \calD}{\phi(x)f^*(x)} - p\cdot\Ex{x \sim \calD}{\phi(x)} \ge \delta$.
    Since $1 - 2p \in [-2\gamma, 2\gamma]$ and $\Ex{x \sim \calD}{\phi(x)} \in [0, 1]$, the second term above is at most $2\gamma$. It follows that the expected error of $\widetilde\phi$ is at most $p + 2\gamma - 2\delta
    \le \left(\frac{1}{2} + \gamma\right) + 2\gamma - 2\delta
    =   \frac{1}{2} - (2\delta - 3\gamma)$.
\end{proof}

\begin{proof}[Proof of \Cref{prop:SQ-refutation-implies-SQ-weak-learning}]
    Let $\calA$ denote the hypothetical SQ algorithm that solves biased-$(\alpha, \eta)$-refutation for $\calC$. Let $\eps, \tau' = \Theta(\tau)$ be sufficiently small such that: (1) $\alpha \le 1/2 - (\eps + 2\tau')$; (2) $\tau \ge 4\eps + 22\tau'$. We construct an SQ algorithm $\calA'$ that weakly learns $\calC$ by simulating $\calA$ on the distribution of $(x, f^*(x))$ where $x \sim \calD$ and $f^* \in \calC$ is the unknown target function in the weak learning instance:
    \begin{itemize}
        \item \textbf{Step 1:} Query the SQ oracle (for the weak learning instance) to estimate the value of $p \coloneqq \Ex{x \sim \calD}{f^*(x)}$ using the constant function $\phi(x) \equiv 1$. Let $\widehat p \in [p - \tau', p + \tau']$ be the output of the oracle. If $\widehat p + \tau' \le 1/2 - \eps$, output the constant function $0$ and terminate. If $\widehat p - \tau' \ge 1/2 + \eps$, output the constant function $1$ and terminate.
        \item \textbf{Step 2:} Simulate the refutation algorithm $\calA$. Whenever $\calA$ tries to query the SQ oracle (for the refutation instance) with test function $\phi: \X \times \zo \to [0, 1]$, consider the function $\Delta: \calX \to [-1, 1]$ defined as $\Delta(x) \coloneqq \phi(x, 1) - \phi(x, 0)$ and $\Delta': \calX \to [0, 1]$ defined as $\Delta'(x) \coloneqq \frac{\Delta(x) + 1}{2}$.
        \item \textbf{Step 3:} Query the SQ oracle (for weak learning) with test function $\Delta'$ to estimate $\mu \coloneqq \Ex{x \sim \calD}{\Delta'(x) \cdot f^*(x)}$. Let $\widehat\mu \in [\mu - \tau', \mu + \tau']$ denote the output of the SQ oracle. Check whether it holds that
        \[
            \left|\widehat\mu - \widehat p \cdot \Ex{x \sim \calD}{\Delta'(x)}\right| \le \frac{\tau}{2} - 2\tau'.
        \]
        If so, we compute $\Ex{x \sim \calD}{\phi(x, 0) + p\cdot \Delta(x)}$, return the result to the refutation algorithm $\calA$, and continue the simulation by going back to Step~2. Otherwise, go to Step~4.
        \item \textbf{Step 4:} We apply \Cref{lemma:rounding} to $\Delta'$ and obtain a randomized boolean function $\widetilde\phi$ from either $\Delta'$ or $1 - \Delta'$. We query the SQ oracle to obtain an estimate $\widehat\eps$ of
        \[
            \pr{x \sim \calD}{\widetilde\phi(x) \ne f^*(x)}
        =   \Ex{x \sim \calD}{\widetilde\phi(x)} + \Ex{x \sim \calD}{f^*(x)} - 2\Ex{x \sim \calD}{\widetilde\phi(x)\cdot f^*(x)}.
        \]
        If $\widehat\eps \le 1/2 - \eps - 3\tau'$, we return the function $\widetilde\phi$. Otherwise, repeat this step.
    \end{itemize}

    If $\calA'$ outputs a constant classifier in the first step, the output clearly has an error $\le 1/2 - \eps$. Thus, we may focus on the case that $|\widehat p - 1/2| \le \eps + \tau'$. Since $|\widehat p - p| \le \tau'$, we must have $|p - 1/2| \le \gamma \coloneqq \eps + 2\tau'$ in this case. The rest of the proof proceeds in the following three steps:
    \begin{itemize}
        \item If $\Delta'$ has a low correlation with $f^*$ (i.e., $\left|\widehat\mu - \widehat p \cdot \Ex{x \sim \calD}{\Delta'(x)}\right| \le \tau - 4\tau'$ holds in Step~3 of $\calA'$), the answer $\Ex{x \sim \calD}{\phi(x, 0) + p \cdot \Delta(x)}$ that we return to $\calA$ is a valid answer for an SQ oracle with tolerance $\tau$.
        \item If $\Delta'$ has a high correlation with $f^*$ (i.e., $\left|\widehat\mu - \widehat p \cdot \Ex{x \sim \calD}{\Delta'(x)}\right| > \tau - 4\tau'$), we will find a good $\widetilde\phi$ without repeating Step~4 too many times.
        \item If $\Delta'$ never has a high correlation with $f^*$, the execution of $\calA$ will be indistinguishable from that in the ``noise'' case of the refutation instance. Therefore, a high-correlation $\Delta'$ must be found with a good probability.
    \end{itemize}

    \paragraph{Low correlation gives accurate answers.} Suppose that, for some test function $\phi: \calX \times \zo \to [0, 1]$ chosen by $\calA$ and the corresponding $\Delta'$, it holds in Step~3 that
    \[
        \left|\widehat\mu - \widehat p \cdot \Ex{x \sim \calD}{\Delta'(x)}\right| \le \frac{\tau}{2} - 2\tau'.
    \]
    Recall that $\widehat\mu$ is within an additive error of $\tau'$ to $\mu = \Ex{x \sim \calD}{\Delta'(x) \cdot f^*(x)}$ and $\widehat p$ is within error $\tau'$ to $p = \Ex{x \sim \calD}{f^*(x)}$. We have
    \[
        \left|\Ex{x \sim \calD}{\Delta'(x) \cdot f^*(x)} - p \cdot \Ex{x \sim \calD}{\Delta'(x)}\right|
    \le \left|\widehat\mu - \widehat p \cdot \Ex{x \sim \calD}{\Delta'(x)}\right| + 2\tau'
    \le  \frac{\tau}{2}.
    \]
    Then, the difference between the correct answer,
    \[
        \Ex{x \sim \calD}{\phi(x, f^*(x))}
    =   \Ex{x \sim \calD}{\phi(x, 0) + (\phi(x, 1) - \phi(x, 0)) \cdot f^*(x)}
    =   \Ex{x \sim \calD}{\phi(x, 0) + \Delta(x) \cdot f^*(x)},
    \]
    and the answer $\Ex{x \sim \calD}{\phi(x, 0) + p\cdot \Delta(x)}$ returned by $\calA'$ is exactly
    \[
        \left|\Ex{x \sim \calD}{\Delta(x)\cdot f^*(x)} - p \cdot \Ex{x \sim \calD}{\Delta(x)}\right|
    =   2 \cdot \left|\Ex{x \sim \calD}{\Delta'(x)\cdot f^*(x)} - p \cdot \Ex{x \sim \calD}{\Delta'(x)}\right|
    \le \tau.
    \]
    In other words, $\calA'$ simulates a valid SQ oracle with tolerance $\tau$ when the correlation is low.

    \paragraph{High correlation gives good $\widetilde\phi$.} Now, suppose that $\left|\widehat\mu - \widehat p \cdot \Ex{x \sim \calD}{\Delta'(x)}\right| > \frac{\tau}{2} - 2\tau'$ holds in Step~3. Again, since $\widehat\mu \approx \mu = \Ex{x \sim \calD}{\Delta'(x) \cdot f^*(x)}$ and $\widehat p \approx p = \Ex{x \sim \calD}{f^*(x)}$ hold up to error $\tau'$, we have
    \[
        \left|\Ex{x \sim \calD}{\Delta'(x) \cdot f^*(x)} - p \cdot \Ex{x \sim \calD}{\Delta'(x)}\right|
    \ge \left|\widehat\mu - \widehat p \cdot \Ex{x \sim \calD}{\Delta'(x)}\right| - 2\tau'
    >   \frac{\tau}{2} - 4\tau'.
    \]
    By the assumption that $\tau \ge 4\eps + 22\tau'$, the above implies $\left|\Ex{x \sim \calD}{\Delta'(x) \cdot f^*(x)} - p \cdot \Ex{x \sim \calD}{\Delta'(x)}\right| \ge \delta \coloneqq 2\eps + 7\tau'$. Recall that we assumed $p \in [1/2 - \gamma, 1/2 + \gamma]$ for $\gamma = \eps + 2\tau'$. Applying \Cref{lemma:rounding} to $\Delta'$ shows that $\Delta'$ can be rounded to a random function $\widetilde\phi: \calX \to \zo$ with an expected error of at most
    \[
        \frac{1}{2} - (2\delta - 3\gamma)
    =   \frac{1}{2} - (\eps + 8\tau').
    \]
    By Markov's inequality, the probability that $\widetilde\phi$ has an error $\le 1/2 - (\eps + 6\tau')$ is at least
    \[
        1 - \frac{1/2 - (\eps + 8\tau')}{1/2 - (\eps + 6\tau')}
    =   \frac{2\tau'}{1/2 - (\eps + 6\tau')}
    = \Omega(\tau).
    \]
    Note that in Step~4, $\widehat\eps$ is within an additive error of $3\tau'$ to the actual error of $\widetilde\phi$. Then, if $\widetilde\phi$ has an error $\le 1/2 - (\eps + 6\tau')$, we would have
    \[
        \widehat\eps
    \le \pr{x \sim \calD}{\widetilde\phi(x) \ne f^*(x)} + 3\tau'
    \le \frac{1}{2} - \eps - 3\tau',
    \]
    and algorithm $\calA'$ would terminate. Therefore, whenever Step~4 is entered, at most $O(1/\tau)$ repetitions are needed in expectation. This shows that $\calA'$ makes at most $O(q + 1/\tau)$ SQs in expectation.

    \paragraph{The probability of making high-correlation queries.} Now, we argue that the hypothetical refutation algorithm $\calA$ must make the aforementioned high-correlation query. To this end, we couple the execution of $\calA$ simulated by our weak learner $\calA'$ (the \emph{simulated copy}) with a slight variant of it (the \emph{imaginary copy}): In the imaginary copy, we never check the correlation or go to Step~4; we always return $\Ex{x \sim \calD}{\phi(x, 0) + p \cdot \Delta(x)}$ for every query $\phi: \calX \times \zo \to [0, 1]$ that $\calA$ makes.

    Note that the imaginary copy of $\calA$ exactly runs on a refutation instance in which the distribution is the product distribution of $\calD$ and $\Bern(p)$, i.e., the label $y$ is a $p$-biased coin flip regardless of $x$. Since we assumed that $|p - 1/2| \le \gamma = \eps + 2\tau' \le 1/2 - \alpha$, we have $p \in [\alpha, 1 - \alpha]$. Then, by the soundness guarantee of $\calA$, the imaginary copy $\calA$ must output $\noise$ with probability at least $2/3$.

    In contrast, the simulated copy of $\calA$ runs on an instance in which the labels are consistent with $f^* \in \calC$. Then, the completeness of $\calA$ ensures that the simulated copy outputs $\structure$ with probability $\ge 2/3$. Therefore, in the coupling between the simulated and imaginary copies, they must diverge with probability at least $1/3$. The only way for the two copies to disagree is that, during the execution of the simulated copy, the algorithm makes an SQ with a high-correlation test function. Therefore, algorithm $\calA'$ outputs a classifier with error $\le 1/2 - \eps$ with probability at least $1/3$. Since $\calA'$ never returns an incorrect answer (i.e., a classifier with error $> 1/2 - \eps$), repeating $\calA'$ a constant number of times would boost the success probability to $2/3$, thereby giving an $\eps$-weak learner for class $\calC$ on distribution $\calD$.
\end{proof}

\subsection{Put Everything Together}
So far, we have established a reduction from SQ weak learning to SQ refutation (\Cref{prop:SQ-refutation-implies-SQ-weak-learning}), and one from SQ refutation to MQ-SQ testable learning (\Cref{prop:MQ-SQ-TLQ-implies-SQ-refutation}). We now combine them and prove a lower bound for MQ-SQ testable learning in terms of the \emph{statistical query dimension} (SQ dimension) of the concept class introduced by Blum, Furst, Jackson, Kearns, Mansour, and Rudich~\cite{BFJKMR94}. 

\begin{definition}[SQ dimension]
    The SQ dimension of a concept class $\calC \subseteq \{f: \calX \to \zo\}$ on a distribution $\calD$ over $\calX$ is the maximum number $d$ such that there exists $f_1, f_2, \ldots, f_d \in \calC$ that satisfy
    \[
        \pr{x \sim \calD}{f_i(x) \ne f_j(x)} \in \left[\frac{1 - 1/d^3}{2}, \frac{1 + 1/d^3}{2}\right]
    \]
    for all $i \ne j \in [d]$.
\end{definition}

\begin{theorem}
\label{thm:MQ-SQ-lower-bound-in-SQ-DIM}
    The following holds for a sufficiently small constant $\eps_0 > 0$: Suppose that $\calC$ is a concept class with SQ dimension $d$ on distribution $\calD$ and $\|\calD\|_2^2 \le O(1 / \poly(d))$. Let $c \ge 1$ and $\eps, \delta \le \eps_0$. Then, no MQ-SQ algorithm can $(c, \eps, \delta)$-testably learn $\calC$ on distribution $\calD$ by making $q \le o(\poly(d))$ queries to an MQ-SQ oracle with tolerance $\tau \ge \omega(1/\poly(d))$ such that $\|\Dstar\|_2^2 \le O(1 / \poly(d))$ holds for all queries of types I~and~II.
\end{theorem}

For concreteness, consider the hypercube $\calX = \zo^n$ and the uniform distribution over it. It is well-known that the family of parity functions has an SQ dimension of $2^n$. Furthermore, for every $k \le n^{1 - \Omega(1)}$, both $k$-juntas and depth-$k$ decision trees have SQ dimensions of $n^{\Omega(k)}$, as they both contain all parity functions of $\le k$ variables. Thus, \Cref{thm:MQ-SQ-lower-bound-in-SQ-DIM} gives $2^{\Omega(n)}$ or $n^{\Omega(k)}$ lower bounds against MQ-SQ testable learners for these classes. Regarding the constraint on $\Dstar$, if $\Dstar$ is the uniform distribution over a $d'$-dimensional subcube, we need $2^{-d'} = \|\Dstar\|_2^2 \le O(1/\poly(d))$, so ensuring $d' = \Omega(n)$ would suffice. (Recall from \Cref{sec:MQ-SQ-examples} that this condition holds when implementing many existing query-based learners as MQ-SQ algorithms.) 

\begin{proof}
    Suppose towards a contradiction that such an MQ-SQ algorithm exists. Ignoring all the other parameters for now, \Cref{prop:MQ-SQ-TLQ-implies-SQ-refutation} shows that there is an algorithm that refutes $\calC$ by making $q' = q + O(1)$ queries to an SQ oracle with tolerance $\tau' = \tau/4$. Applying \Cref{prop:SQ-refutation-implies-SQ-weak-learning} then gives an algorithm that $\Omega(\tau)$-weakly learns parity functions using $O(q' + 1/\tau') = O(q + 1/\tau)$ queries to an SQ oracle with tolerance $\Theta(\tau)$. By~\cite[Theorem 12]{BFJKMR94}, to $\Omega(1/d^3)$-weakly learn concept class $\calC$ using an SQ oracle with tolerance $\Omega(d^{-1/3})$, at least $\Omega(d^{1/3})$ queries are needed. Since $q, 1/\tau = o(\poly(d))$, we obtain a contradiction.
    
    Now, we set the parameters in \Cref{prop:MQ-SQ-TLQ-implies-SQ-refutation,prop:SQ-refutation-implies-SQ-weak-learning} carefully. We may assume that $\tau \le \eps_0 / c$ without loss of generality; a smaller tolerance makes the SQ oracle (and thus the lower bound result) stronger. Since $\eps, \delta, c\tau \le \eps_0$ are sufficiently small and $\tau' = \tau/4$, we can choose $\alpha = 0.1$ and $\eta = 0$ in \Cref{prop:MQ-SQ-TLQ-implies-SQ-refutation} such that the first condition $\alpha > c\eta + \eps + (c + 4)\tau + 6\tau'$ is satisfied. Furthermore, the condition that $\alpha \le 1/2 - \Omega(\tau)$ in \Cref{prop:SQ-refutation-implies-SQ-weak-learning} would also hold. Recall that the failure probability increases from $\delta \le \eps_0$ to $\delta + O(q) \cdot e^{-\Omega(\tau^2B)}$ in \Cref{prop:MQ-SQ-TLQ-implies-SQ-refutation}. Setting $B = \Theta((\log q) / \tau^2) = o(\poly(d))$ suffices to control the new failure probability by $2\eps_0$.
    
    It remains to check the second and the third conditions of \Cref{prop:MQ-SQ-TLQ-implies-SQ-refutation}. For the third, we need the MQ-SQ testable learner to restrict the distribution $\Dstar$ in its queries such that $\|\Dstar\|_2^2 \le 1/B$. This is ensured by $\|\Dstar\|_2^2 \le O(1/\poly(d))$ and $B \le o(\poly(d))$. Finally, to check the second condition that $B \le p^2(1-p)^2/\|\calD_x\|_2^2$, we note that $\|\calD_x\|_2^2 = \|\calD\|_2^2 \le O(1/\poly(d))$. Furthermore, by the way in which the reduction works in the proof of \Cref{prop:SQ-refutation-implies-SQ-weak-learning}, whenever the refutation algorithm is called, the labels are nearly balanced, i.e., $p^2(1-p)^2 = \Omega(1)$. Therefore, the second condition is always satisfied by our choice of $B = o(\poly(d))$. This completes the proof.
\end{proof}

\bibliographystyle{alpha}
\bibliography{main}

\end{document}